\documentclass[11pt]{article}

\usepackage[margin=1in]{geometry}
\usepackage{amsmath,amssymb,mathtools,amsthm}
\usepackage{lmodern}
\usepackage{bm}
\usepackage{booktabs}
\usepackage{microtype}
\usepackage{xcolor}
\usepackage{graphicx}
\usepackage{multirow}
\usepackage{enumitem}
\usepackage[round,authoryear]{natbib}
\usepackage[colorlinks=true,citecolor=blue,linkcolor=blue,urlcolor=blue]{hyperref}
\hypersetup{
  pdftitle={Beyond Negative-Ridge Endpoints: Mixed-Sign Spectral Regularization via Negative-Shifted Gradient Descent},
  pdfauthor={Peng Zhao},
  pdfsubject={Implicit regularization in overparameterized linear regression},
  pdfkeywords={negative-shifted gradient descent, mixed-sign spectral regularization, implicit regularization, early stopping, benign overfitting}
}


\theoremstyle{plain}
\newtheorem{theorem}{Theorem}
\newtheorem{proposition}[theorem]{Proposition}
\newtheorem{lemma}[theorem]{Lemma}
\newtheorem{corollary}[theorem]{Corollary}
\theoremstyle{definition}
\newtheorem{assumption}[theorem]{Assumption}
\theoremstyle{remark}
\newtheorem{remark}[theorem]{Remark}

\newenvironment{keywords}
{\par\smallskip\noindent\textbf{Keywords:}\ }
{\par\medskip}

\newcounter{algorithm}
\renewcommand{\thealgorithm}{\arabic{algorithm}}

\newcommand{\R}{\mathbb{R}}
\newcommand{\E}{\mathbb{E}}
\newcommand{\Risk}{\mathcal{R}}
\newcommand{\tr}{\operatorname{tr}}

\begin{document}

\title{Beyond Negative-Ridge Endpoints: Mixed-Sign Spectral Regularization
via Negative-Shifted Gradient Descent}

\author{Peng Zhao\\
\small Department of Applied Economics and Statistics\\
\small University of Delaware\\
\small \href{mailto:pzhao@udel.edu}{\texttt{pzhao@udel.edu}}}
\date{July 2026}

\maketitle

\begin{abstract}%
In overparameterized linear regression, many weak spectral directions act
like a ridge penalty on the signal-bearing spectrum; negative ridge is the
natural correction, pushing filters above one.  The stable negative-ridge
endpoint, however,
is structurally limited: its pole must stay below the smallest empirical
eigenvalue, and it anti-shrinks smaller eigenvalues more than larger ones.
Early-stopped negative-shifted gradient descent escapes this constraint.  Its
filter is smooth at the would-be pole and
mixed-sign-capable: above-ridgeless directions form a leading
prefix, with lower directions shrunk or exposure-controlled while stopping
sets the crossover.  In a Gaussian spike-plus-flat model we discover a
Marchenko--Pastur barrier: the shift that cancels the implicit penalty lies
a bulk width above the smallest empirical eigenvalue, and the stopped path
improves on every admissible endpoint by a polynomial factor in risk under
explicit conditions.  Our
main theorem permits a general high-effective-rank tail: its trace sets the
implicit floor, its squared spectrum controls exposure, and the
floor-critical path recovers all head scales at once---beyond positive
shrinkage and, once scales separate, every uniform rescaling of ridgeless.
Handling the noncontractive shifted dynamics is the central technical
challenge; localized Duhamel integrals control them.
A finite-grid hold-out inequality transfers the separations to the
validation-selected algorithm.
\end{abstract}


\section{Introduction}

Ridge regression and early-stopped gradient descent are two of the most common
regularization tools in high-dimensional linear regression.  Ridge gives an
explicit spectral shrinkage rule, while early stopping regularizes through the
time profile of the optimization path
\citep{hoerl1970ridge,Yao2007Early,Dobriban2018High-dimensional,Ali2019Continuous,Tsigler2023Benign}.
In overparameterized and anisotropic problems, however, the usual shrinkage
picture is incomplete.  A large collection of weak spectral directions can
stabilize interpolation, but it can also make the signal-bearing part of the
spectrum behave as if it were already ridge-penalized
\citep{Bartlett2020Benign,Hastie2022Surprises,Tsigler2023Benign}.

This implicit shrinkage can be strong enough that the best ridge correction is
negative.  In the head--tail view of benign-overfitting ridge theory, many weak
tail directions create an effective positive floor on the head
coordinates; canceling this floor can require a negative ridge parameter
\citep{Tsigler2023Benign}.  Negative ridge is therefore a natural response to
over-shrinkage: it pushes spectral filters above one on directions where the
ridgeless interpolator under-estimates the signal.

For a signed level \(\nu\ge0\), the negative-ridge objective is
\(\mathcal L_\nu(\beta)
:=\frac{1}{2n}\|y-X\beta\|_2^2-\frac{\nu}{2}\|\beta\|_2^2\).
When \(p>n\), \(\mathcal L_\nu\) is unbounded along the null space of \(X\).
Its stable endpoint is therefore understood on the empirical row space and
exists only for \(0\le\nu<\widehat\mu_{\min}^+\).  The finite-time path from
zero also remains in the row space, but it does not require convergence to
this endpoint and may use larger signed levels.

The correction we seek is not uniform amplification.  It should raise a
leading, signal-rich prefix of directions above the ridgeless level while
keeping the lower spectrum shrunk or exposure-controlled---a displacement from
ridgeless that changes sign once, rather than in an arbitrary mode-by-mode
pattern.  The stable endpoint supplies the sign but not this shape: its filter
\(A_\nu(\mu)=\mu/(\mu-\nu)\) has above-one displacement
\(A_\nu(\mu)-1=\nu/(\mu-\nu)\), which \emph{decreases} with \(\mu\), so within
its admissible range the rational shape couples any head correction to larger
anti-shrinkage---and hence variance---on the lower stable spectrum.

We therefore propose \emph{negative-shifted gradient descent} (NS-GD, the
discrete algorithm) and its continuous-time flow limit (NS-GF) as a
two-parameter spectral regularization path.
For a negative ridge parameter \(\lambda_{\rm reg}=-\nu\), with \(\nu\ge0\),
the negative-shifted gradient flow from zero initialization satisfies
\[
\dot\beta_t
=
n^{-1}X^\top(y-X\beta_t)+\nu\beta_t
=
b-(\widehat\Sigma-\nu I)\beta_t,
\qquad
\beta_0=0,
\]
where \(b=X^\top y/n\) and \(\widehat\Sigma=X^\top X/n\).  The resulting
spectral filter is \(f_{\nu,t}\), defined formally in
\eqref{eq:signed-gf-filter}, with removable
value \(f_{\nu,t}(\nu)=\nu t\); the same location that is a pole for the
endpoint becomes a removable singularity for the finite-time path.

This removable-pole filter is what lets NS-GD realize the shape the endpoint
cannot.  By Corollaries~\ref{cor:signed-gf-superlevel-prefix}
and~\ref{cor:signed-gd-superlevel-prefix}, its displacement
from ridgeless changes sign at a single spectral threshold, so a leading head
prefix is anti-shrunk while the lower spectrum stays shrunk or
exposure-controlled, with the crossover set by the stopping time.  This is the
sense in which NS-GD is \emph{mixed-sign-capable}, the property that separates it
from the comparators: positive ridge and early stopping never cross the ridgeless
level, and the stable endpoint crosses it only with the tail-heavy pole shape
above, whereas the finite-time path crosses on a controlled prefix and may use
shifts outside the stable range, provided it stops before the lower-spectrum
transient becomes substantial.
Full-batch discrete gradient descent has the analogous finite geometric-sum
filter.
Algorithm~\ref{alg:negative-shifted-gd} gives the validation-selected
finite-horizon implementation used in the experiments.  The step size is a
stability and discretization parameter; the signed level and finite horizon are
the regularization coordinates.

\begin{figure}[t]
\refstepcounter{algorithm}%
\label{alg:negative-shifted-gd}%
\centering
\fbox{%
\begin{minipage}{0.93\linewidth}
\textbf{Algorithm \thealgorithm: Validation-selected negative-shifted GD
(NS-GD)}

\smallskip
\textbf{Input:} training data \((X_{\rm tr},y_{\rm tr})\), validation data
\((X_{\rm val},y_{\rm val})\), a finite signed-level grid
\(\mathcal V\subset[0,\infty)\), maximum iteration \(M\), validation
times \(\mathcal T\subseteq\{0,\ldots,M\}\), and training-measurable step
sizes \(\{\eta_\nu\}\) obeying the small-step condition
\(\eta_\nu(\widehat\mu_1^{\rm tr}+\nu)\le c_\eta<1\), where
\(\widehat\mu_1^{\rm tr}
=\lambda_{\max}(X_{\rm tr}X_{\rm tr}^{\top}/n_{\rm tr})\).

\begin{enumerate}[leftmargin=1.6em,itemsep=2pt,topsep=3pt]
\item For every \(\nu\in\mathcal V\): initialize \(\beta_{\nu,0}=0\) and, for
\(m=0,\ldots,M-1\), run the full-batch update
\[
\beta_{\nu,m+1}
=
\beta_{\nu,m}
+\eta_\nu\left\{
\frac{X_{\rm tr}^{\top}
(y_{\rm tr}-X_{\rm tr}\beta_{\nu,m})}{n_{\rm tr}}
+\nu\beta_{\nu,m}
\right\}.
\]
\item Record
\(L(\nu,m)=\|y_{\rm val}-X_{\rm val}\beta_{\nu,m}\|_2^2/n_{\rm val}\)
at each \(m\in\mathcal T\), select
\((\widehat\nu,\widehat m)
\in\arg\min_{\mathcal V\times\mathcal T}L(\nu,m)\),
and return
\(\widehat\beta_{\rm NS\text{-}GD}=\beta_{\widehat\nu,\widehat m}\).
\end{enumerate}

\textit{Finite-exposure distinction.}
\(\mathcal V\) may include \(\nu\ge\widehat\mu_{\min}^+\): the algorithm
evaluates finite iterates and never forms the unstable inverse
\((\widehat\Sigma-\nu I)^{-1}\); growth below \(\nu\) is controlled by the
finite horizon (for a target continuous time \(t\), take \(\eta_\nu=t/m\)
with \(m\) large).
\end{minipage}}
\end{figure}

The finite-grid choice in Algorithm~\ref{alg:negative-shifted-gd} is not an
oracle choice of \((\nu,m)\): Theorem~\ref{thm:finite-grid-validation-oracle}
gives a finite-sample hold-out inequality under the same independent Gaussian
design model used by the random-design theory.  The guarantee is conditional
on a fixed training-measurable step-size rule and a predeclared finite grid.

This endpoint--path asymmetry differs from the familiar positive-side
advantage of early stopping.  With nonnegative regularization, ridge and
early-stopped GD are both contraction filters; their differences arise from
residual decay and qualification or saturation within the shrinkage regime
\citep{Engl1996Regularization,Yao2007Early}, and recent instance-wise theory
shows GD within a constant factor of comparably regularized ridge while tuned
ridge can be polynomially worse \citep{wu2025risk}.  The negative side has a
different obstruction: random matrix geometry pins the endpoint's pole against
the lower empirical edge.

The random-design comparison below quantifies this restriction.  In a
common-spike-plus-flat model with tail floor \(a=\gamma_T\lambda_T\) and
common head scale \(\lambda_h\), the lower empirical edge lies a bulk-width
distance below the floor, so an admissible endpoint must retreat from \(a\) by
order \(\sqrt{a\lambda_T}\).  The stopped path instead uses the
strictly supercritical level and finite time
$
\nu_\star=a+\lambda_h,
t_\star=\lambda_h^{-1},
$
placing its removable pole at the floor-shifted common head location: this
exactly cancels the implicit head shrinkage in the ideal floor model while
exposing the tail for only a constant rescaled time, and the discrete choice
\(\eta_\star=(m\lambda_h)^{-1}\) gives the same recovery after any \(m\ge1\)
steps (Theorem~\ref{thm:journal-spike-flat-separation}).

The common-spike calculation has a degenerate feature: because every head
coordinate has the same effective attenuation, a uniformly rescaled ridgeless
estimator can reproduce its head correction.
Section~\ref{sec:multispike-floor-critical} therefore introduces distinct
spike levels and a heterogeneous high-effective-rank tail.  The tail trace
\(T_k\) sets the floor \(a=T_k/n\), while the squared spectrum \(S_k\)
controls finite exposure and endpoint variance.  Centering the signed
dynamics at \(\nu_{\rm fc}=a\) makes the ideal head map
\(I-e^{-tG_H}\), which recovers several resolved scales simultaneously.
A weighted shape defect quantifies why no scalar rescaling can do the same.
Theorem~\ref{thm:journal-multispike-floor-critical} gives the resulting
TB-style Gaussian bound at the theoretically chosen floor-critical level
\(\nu_{\rm fc}=a=T_k/n\); flat and power-law tails are explicit corollaries.

Figure~\ref{fig:motivation-supercritical} illustrates this endpoint--path
asymmetry at a fixed negative shift: the endpoint's displacement grows toward
its pole down the spectrum, while the finite-time filter is smooth at the same
location and crosses ridgeless on a leading prefix.  The analysis below
controls the finite-time path through variance, leakage, and transfer
budgets.

\begin{figure}[t]
	\centering
	\includegraphics[width=0.8\linewidth]{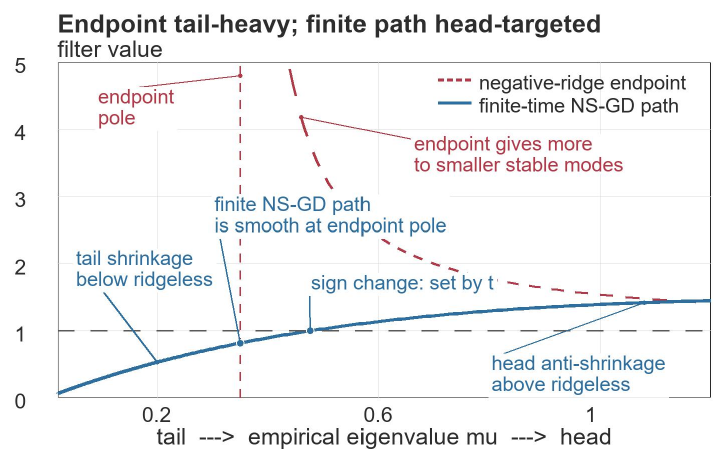}
	\caption{Endpoint--path asymmetry for mixed-sign spectral regularization.
		A stable negative-ridge endpoint is tail-heavy---limited by its pole and
		amplifying smaller stable eigenvalues most---whereas the finite-time
		NS-GD filter is smooth at the same spectral location and crosses the
		ridgeless level on a possibly full leading spectral prefix.  When the
		crossover lies in the observed spectrum, lower modes remain shrunk or
		exposure-controlled; the stopping time sets that crossover.}
	\label{fig:motivation-supercritical}
\end{figure}

\noindent\textbf{Positioning.}
Three lines of work are directly relevant.  First, early stopping and
iterative regularization have long been understood as spectral regularization
methods, closely related to Tikhonov and ridge-type procedures
\citep{Engl1996Regularization,LoGerfo2008Spectral,Yao2007Early,Bauer2007Regularization,caponnetto2007optimal,Raskutti2011Early,hu2022early}.
Recent work further characterizes discrete full-batch GD as a generalized
ridge procedure and compares gradient-flow, gradient-descent, and conjugate
gradient regularization paths with ridge
\citep{SonthaliaLokRebrova2024EarlyStopping,HuckerReissStark2026Paths};
related equivalence results connect explicit ridge tuning with subsampling and
ensemble regularization \citep{PatilDu2023Generalized}, while complementary
work studies data-driven risk estimation and tuning for early-stopped GD
\citep{PatilWuTibshirani2024GD}.
These works clarify how stopping time trades bias and variance within the
positive shrinkage regime.  Our setting differs because the path is
negative-shifted: finite time creates a mixed-sign filter class when the
negative endpoint itself has a pole.

Second, the paper connects to implicit regularization in least squares and
high-dimensional regression.  Gradient-flow and stochastic-gradient analyses
compare optimization paths with explicit \(\ell_2\) regularization
\citep{Ali2019Continuous,ali2020implicit,zou2021benefits}; the instance-wise
GD-versus-ridge comparison of \citet{wu2025risk} likewise stays within the
contraction family.  Our focus is the signed side, where the stable endpoint
meets the MP pole barrier and the finite-time path becomes a mixed-sign
implicit regularizer.
Random-matrix approaches have also produced deterministic high-dimensional
risk trajectories for stochastic-gradient dynamics in linear models
\citep{PaquetteLeePedregosaPaquette2021SGD,
PaquettePaquetteAdlamPennington2022Homogenization,
AtanasovBordelonZavatoneVethPaquettePehlevan2025TwoPoint}.  Those results
analyze stochastic optimization trajectories; our comparison concerns the
admissibility wall for a negative-ridge endpoint and the finite-time
full-batch path at the same signed shift.

Third, our results build on benign overfitting and overparameterized linear
regression.  This literature characterizes how interpolation, ridge
regularization, and spectral effective ranks determine prediction risk
\citep{Bartlett2020Benign,Hastie2022Surprises,Dobriban2018High-dimensional,Tsigler2023Benign,Nakkiran2021Optimal,Liang2020Just,Cheng2024Dimension}.
Source-condition analyses further clarify how the signal's alignment with the
covariance spectrum controls ridge and ridgeless behavior
\citep{RichardsMourtadaRosasco2021RidgelessSource}.
Negative ridge has appeared as a response to implicit over-regularization:
\citet{Dobriban2018High-dimensional,KobakLomondSanchez2020Optimal} explain
negative optimal ridge through high-dimensional asymptotics and low-variance
covariates, \citet{Tsigler2023Benign} develop ridge theory in regimes where
the optimal regularization can be negative, and
\citet{WuXu2020OptimalWeighted} characterize signs of optimal weighted
\(\ell_2\) regularization in proportional limits.  Those works characterize
\emph{when} the population-optimal ridge is negative and its sign.  We instead
isolate the endpoint--path asymmetry: the endpoint can take the useful sign,
but not the floor-canceling \emph{magnitude}, which lies beyond the sample
Marchenko--Pastur pole and is reachable only along the finite-time path.

\noindent\textbf{Contributions.}
Our contributions are as follows.
\begin{enumerate}
	\item \textbf{Endpoint--path asymmetry and the Marchenko--Pastur pole
	barrier.}
	Positive filters are contractions (\(f\le1\)), so their only freedom is
	saturation or residual profile.  The negative-ridge endpoint carries a pole
	of the \emph{useful} sign but is admissible only below the smallest
	empirical eigenvalue; in the flat-tail benchmark the tail floor sits a bulk
	width \(a-\widehat\mu_{\min}^+\asymp\sqrt{a\lambda_T}\) above that edge, so
	the endpoint can take the useful sign but not the floor-canceling
	\emph{magnitude}, which finite-time NS-GD reaches through its removable
	pole.

	\item \textbf{Constrained-prefix anti-shrinkage.}
	In the exact design-conditional risk geometry of
	Section~\ref{sec:exact-conditional-target}, anti-shrinkage in a direction
	helps exactly when its missing-signal coupling exceeds the variance cost
	(\(m_i>v_i\)), so the oracle asks for a selective correction.  NS-GD
	realizes a leading-prefix-constrained form of this rule: by
	Corollaries~\ref{cor:signed-gf-superlevel-prefix}
	and~\ref{cor:signed-gd-superlevel-prefix} the
	displacement from ridgeless changes sign at a single spectral
	threshold---in strong-signal realizations the anti-shrunk prefix can be the
	whole observed spectrum---a data-adaptive prefix geometry unavailable to
	positive shrinkage and to the pole-limited endpoint.

	\item \textbf{Risk separations for common and heterogeneous heads, via a
	localized noncontractive proof.}
	In the common-spike model the explicit continuous and discrete paths
	recover the spike with risk \(O(\lambda_T/\lambda_h)\), while admissible
	negative ridge pays \(\Omega(\sqrt{\lambda_T/\lambda_h})\) and nonnegative
	shrinkage retains a constant head-bias floor.  The main random-design
	theorem replaces the common spike by the trace floor \(a=T_k/n\) and
	recovers distinct head scales simultaneously: the square mass \(S_k\)
	controls tail exposure, leakage, and the negative-endpoint lower envelope,
	and a weighted shape defect lower-bounds \emph{every} uniform rescaling of
	ridgeless regression; flat and truncated power-law tails follow as
	corollaries.  The proofs turn on a localized step: the shifted generator
	can be indefinite, so we align the realized clusters with an orthogonal
	invariant graph and apply Duhamel only inside the positive head block,
	whereby head--tail coupling must leave the head and return and enters
	quadratically---letting the Tsigler--Bartlett trace/effective-rank
	architecture accommodate the noncontractive signed path.

	\item \textbf{Deployable validation oracle and experiments.}
	For the finite grid in Algorithm~\ref{alg:negative-shifted-gd}, an
	independent hold-out sample selects a candidate whose realized prediction
	risk is within an explicit multiplicative factor of the best grid
	candidate, so every separation above transfers to the deployed estimator
	under an explicit validation-size condition.  Paired simulations diagnose
	the predicted endpoint wall and the adaptive finite-time spectral shape.
\end{enumerate}

\noindent\textbf{Roadmap.}
Section~\ref{sec:exact-conditional-target} presents the exact
design-conditional oracle, Section~\ref{sec:gaussian-spike-models} the two
Gaussian head--tail regimes and their proofs, and
Theorem~\ref{thm:finite-grid-validation-oracle} the transfer to the
validation-selected estimator; experiments and discussion follow.

\section{Row-space filters, risk, and finite-time geometry}
\label{sec:model}

Let \(X\in\R^{n\times p}\) be fixed for the spectral calculations, with
\(\widehat\Sigma=X^\top X/n\) and nonzero eigenpairs
\((\widehat\mu_i,\widehat v_i)_{i=1}^r\). The response is
\(y=X\beta^\star+\varepsilon\), with
\(\E(\varepsilon\mid X)=0\) and
\(\operatorname{Var}(\varepsilon\mid X)=\sigma_\varepsilon^2 I_n\).  Write
\(b=X^\top y/n\) and \(\theta_i^\star=\widehat v_i^\top\beta^\star\). For any
spectral filter \(f=(f_i)_{i=1}^r\), define
\(\widehat v_i^\top\widehat\beta_f=(f_i/\widehat\mu_i)\widehat v_i^\top b\).
The conditional empirical-spectrum prediction risk is
\begin{equation}
\label{eq:risk-decomposition}
\Risk(f)
=
\sum_{i=1}^r \widehat\mu_i(1-f_i)^2(\theta_i^\star)^2
+
\frac{\sigma_\varepsilon^2}{n}\sum_{i=1}^r f_i^2 .
\end{equation}
The first term is squared prediction bias and the second term is conditional
variance. The filters below are row-space filters; null-space components remain
zero from zero initialization. Let \(\widehat V=(\widehat v_1,\ldots,\widehat v_r)\).
For generalization, the relevant quantity is population prediction risk.  If
\((x_{\rm new},y_{\rm new})\) is an independent test point with
\(\mathbb E x_{\rm new}x_{\rm new}^\top=\Sigma\), then
\(\mathbb E_\varepsilon\|\widehat\beta_f-\beta^\star\|_\Sigma^2\) is the
conditional excess test error over the irreducible noise.
\begin{proposition}[Conditional population risk]
\label{prop:conditional-pop-risk}
Conditional on \(X\), every empirical spectral filter \(f\in\R^r\) satisfies
\[
\Risk_{\rm pop}(f\mid X)
=
\left\|
\left(
\widehat V\operatorname{diag}(f)\widehat V^\top-I
\right)\beta^\star
\right\|_\Sigma^2
+
\frac{\sigma_\varepsilon^2}{n}
\sum_{i=1}^r
f_i^2
\frac{\widehat v_i^\top\Sigma\widehat v_i}{\widehat\mu_i},
\]
where \(\|z\|_\Sigma^2=z^\top\Sigma z\).
This is the finite-sample, design-conditional bias--variance identity
underlying the ridgeless prediction-risk calculations of
\citet{Hastie2022Surprises} and the weighted-\(\ell_2\) risk analysis of
\citet{WuXu2020OptimalWeighted}.
\end{proposition}

Negative \(\ell_2\) regularization is interpreted row-space spectrally. In the
overparameterized case, the primal objective with a negative penalty is
unbounded along the null space of \(X\). Starting from zero, however, gradient
flow and discrete gradient descent remain in the row space. All signed ridge
and signed-ridge path estimators in this paper are therefore defined on the
empirical nonzero spectrum, equivalently through the dual or row-space
representation.

\subsection{Mixed-sign filter geometry: endpoint poles versus finite-time smoothness}
\label{sec:mechanism}

The endpoint--path comparison uses three elementary filter facts.  Positive
paths are shrinkage filters; stable signed endpoints have rational poles;
finite-time negative-shifted gradient flow is smooth at the pole and crosses
above \(f=1\) only on a leading spectral prefix, giving head anti-shrinkage with
the lower modes shrunk or exposure-controlled.  Same-shift endpoint/path
comparisons are useful diagnostics and are deferred to
Appendix~\ref{app:fixed-design-proofs}.  The main comparison below is
class-level.
We use the classical spectral-filter language of inverse problems and
supervised learning \citep{Engl1996Regularization,LoGerfo2008Spectral}, but
allow a signed shift whose finite-time filter is no longer a contraction.

\begin{proposition}[Filter geometry of negative-shifted gradient flow]
\label{prop:filter-geometry-signed-finite-time}
A stable scalar signed-ridge endpoint and negative-shifted gradient flow
(NS-GF) have the filters
\begin{align}
A_\nu(\mu)
&=
\frac{\mu}{\mu-\nu}
\qquad(\mu>\nu),
\label{eq:signed-endpoint-filter}\\
f_{\nu,t}(\mu)
&=
\frac{\mu}{\mu-\nu}
\left(1-e^{-(\mu-\nu)t}\right)
=
\mu\int_0^t e^{-(\mu-\nu)s}\,ds.
\label{eq:signed-gf-filter}
\end{align}
The endpoint has a pole at \(\mu=\nu\); by the integral form, the path's
singularity there is removable, \(f_{\nu,t}(\nu)=\nu t\), with
\(|f_{\nu,t}(\mu)|\le\mu t e^{\nu t}\) for finite \(t\).
Detailed scalar comparisons are in Appendix~\ref{app:fixed-design-proofs}.
\end{proposition}

\begin{corollary}[Above-one filters form a spectral prefix]
\label{cor:signed-gf-superlevel-prefix}
For \(\nu>0\), define
\[
g_\nu(\mu)
=
\begin{cases}
\log(\mu/\nu)/(\mu-\nu), & \mu\ne\nu,\\[2pt]
1/\nu, & \mu=\nu ,
\end{cases}
\qquad\text{so that}\qquad
f_{\nu,t}(\mu)>1
\;\Longleftrightarrow\;
t>g_\nu(\mu),
\]
with \(g_\nu\) strictly decreasing on \((0,\infty)\).
Consequently, after sorting the nonzero empirical eigenvalues decreasingly,
the set of empirical directions with \(f_{\nu,t}(\widehat\mu_i)>1\) is a
spectral prefix.  The proof is in Appendix~\ref{app:fixed-design-proofs}.
\end{corollary}

\noindent
The \(m\)-step Euler response and filter are
\[
h_{\nu,\eta,m}^{\rm disc}(\mu)
:=
\eta\sum_{\ell=0}^{m-1}\{1-\eta(\mu-\nu)\}^{\ell},
\qquad
f_{\nu,\eta,m}^{\rm disc}(\mu)
:=
\mu h_{\nu,\eta,m}^{\rm disc}(\mu).
\]

\begin{corollary}[Small-step NS-GD has the same prefix geometry]
\label{cor:signed-gd-superlevel-prefix}
Let \(\eta,\nu>0\), \(m\ge1\), and suppose
\(\eta(\widehat\mu_1+\nu)<1\).  Then the empirical directions with
\(f_{\nu,\eta,m}^{\rm disc}(\widehat\mu_i)>1\) form a possibly empty or full
leading spectral prefix.  In particular, Algorithm~\ref{alg:negative-shifted-gd}
obeys this geometry.
\end{corollary}

\noindent\textbf{Endpoint shape versus finite-time shape.}
Corollaries~\ref{cor:signed-gf-superlevel-prefix}
and~\ref{cor:signed-gd-superlevel-prefix} give the key shape
distinction.  The endpoint's displacement
\(A_\nu(\mu)-1=\nu/(\mu-\nu)\) is \emph{increasing} toward the pole, so any
head anti-shrinkage forces at least as much on every lower stable
direction---exactly where the conditional variance factor
\(f_i^2\widehat v_i^\top\Sigma\widehat v_i/\widehat\mu_i\) makes amplification
most costly.  The finite-time filter reverses this order: its above-one
directions form a leading prefix, so high directions can be anti-shrunk while
lower directions stay below one, provided the path stops before their
transient grows.  Stable endpoints are thus constrained by a rational pole
shape, while finite-time paths are governed by variance, leakage, and
transfer budgets.

\begin{remark}[Mixed-sign spectral regularization]
\label{rem:effective-scale-signed-path}
The NS-GF filter in \eqref{eq:signed-gf-filter} can be written as
\(f_{\nu,t}(\mu)=\mu/\{\mu+x^-_{\rm eff}(\mu,t)\}\) with eigenvalue-dependent
effective scale
\(x^-_{\rm eff}(\mu,t)=-\nu+(\mu-\nu)/\{e^{(\mu-\nu)t}-1\}\) and continuous
extension \(x^-_{\rm eff}(\nu,t)=1/t-\nu\)
(Lemma~\ref{lem:finite-time-effective-spectral-ridge}).  Unlike the constant endpoint scale
\(-\nu\), this effective scale varies with the realized sample eigenvalue
\(\mu\).  The displacement \(f_{\nu,t}(\mu)-1\) is not an arbitrary sign
pattern: by Corollary~\ref{cor:signed-gf-superlevel-prefix} it is positive
exactly on a leading spectral prefix, so ``mixed-sign'' here means head
anti-shrinkage with the lower spectrum shrunk or exposure-controlled below a
single crossover set by the stopping time; in the strong-signal case the
crossover falls below the observed spectrum and the whole spectrum is
anti-shrunk.  
\end{remark}

In summary, the comparator classes act through different spectral objects.
Every positive comparator---ridge, PCR, and positive gradient paths---stays at
or below the ridgeless amplitude \(f=1\); the negative-ridge endpoint can exceed
it, but only through an admissibility-limited pole below the empirical edge.
Only NS-GF and its small-step NS-GD discretization cross the ridgeless line on
a controlled leading prefix, replacing the endpoint pole by the removable value
\(f_{\nu,t}(\nu)=\nu t\).

\subsection{Exact full-risk oracle geometry}
\label{sec:exact-conditional-target}

Let \(r=\operatorname{rank}(X)\).  For the nonzero empirical eigenpairs of
\(X^\top X/n\), write \(P_i=\widehat v_i\widehat v_i^\top\),
\(P=\sum_{i=1}^rP_i\), and \(P_0=I-P\).  For each empirical row-space
direction define \(\theta_i^\star=\widehat v_i^\top\beta^\star\) and
\(G_{ij}=\widehat v_i^\top\Sigma\widehat v_j\).  Set
\[
H_{ij}:=\theta_i^\star\theta_j^\star G_{ij},
\qquad
m_i:=\beta^{\star\top}P_i\Sigma P_0\beta^\star,
\qquad
v_i:=\frac{\sigma_\varepsilon^2}{n}\frac{G_{ii}}{\widehat\mu_i}.
\]
Thus \(H\) is the full signal-curvature matrix in population prediction
geometry, while \(v_i\) is the variance price of increasing empirical direction
\(i\).  Put \(D=\operatorname{diag}(v_1,\ldots,v_r)\), \(Q=H+D\), and
\(g=m-v\).  Every empirical spectral filter \(f\) is represented by its
displacement from ridgeless, \(\delta_i(f)=f(\widehat\mu_i)-1\).
The methods remain diagonal filters in the empirical eigenbasis, but their
population risk is evaluated in the full \(Q\)-geometry because \(\Sigma\) can
fail to commute with \(X^\top X/n\).

The ridgeless reference is the minimum-norm interpolator whose proportional
random-design prediction risk is characterized by
\citet{Hastie2022Surprises}.  The following result is instead an exact
finite-sample identity conditional on \(X\).  It is elementary
completing-the-square algebra in the displacement coordinates, and we do not
claim it as new; we record it because it fixes the geometry---the
\(Q\)-metric and the target \(\delta^\star\)---in which every comparison
below is evaluated.

\begin{proposition}[Exact full-risk oracle geometry and local anti-shrinkage]
\label{prop:full-risk-oracle-geometry}
For any empirical spectral filter \(f\),
\begin{equation}
\label{eq:full-risk-quadratic}
\Risk_{\rm pop}(f\mid X)
=
\Risk_{\rm pop}(\mathbf 1\mid X)
+
\delta(f)^\top Q\delta(f)
-
2g^\top\delta(f).
\end{equation}
Here \(\mathbf 1\) denotes the ridgeless row-space filter.
The minimum-norm full conditional oracle is
\(\delta^\star:=Q^\dagger g\); when \(Q\) is singular, the full oracle set is
\(\delta^\star+\operatorname{null}(Q)\).  The risk
completes the square as
\begin{equation}
\label{eq:full-risk-square}
\Risk_{\rm pop}(f\mid X)
=
\Risk_{\rm pop}(\mathbf 1\mid X)-g^\top Q^\dagger g
+
\|\delta(f)-\delta^\star\|_Q^2,
\qquad
\|u\|_Q^2:=u^\top Qu,
\end{equation}
For a one-coordinate perturbation of ridgeless, \(f_i=1+\epsilon\) with all
other row-space filter values equal to one,
\[
\left.
\frac{d}{d\epsilon}\Risk_{\rm pop}(f\mid X)
\right|_{\epsilon=0}
=
-2(m_i-v_i).
\]
Thus anti-shrinkage in direction \(i\) is locally risk-decreasing if and only if
\(m_i>v_i\).
\end{proposition}

The target \(\delta^\star\) serves as a population-risk oracle, explaining why
a validation-selected signed path can be closer to the desired spectral shape
than any stable scalar endpoint.  The oracle \(1+\delta^\star\) lives in the diagonal filter
coordinates of the method class, but distance to it is measured in the full
\(Q\)-metric, including off-diagonal population-covariance couplings.  The
Moore--Penrose inverse handles possible degeneracy of this metric; the range
condition needed for the completed square is verified in the proof.

\noindent\textbf{Benefit--cost interpretation.}
This local criterion admits a natural benefit--cost reading.  Although
\(P_iP_0=0\) in Euclidean geometry, generally \(P_i\Sigma P_0\ne0\).  The term
\(m_i=\langle P_i\beta^\star,P_0\beta^\star\rangle_\Sigma\) measures whether
increasing row-space direction \(i\) can compensate for signal left in the
empirical null space.  The term \(v_i\) is the variance price.  If risk were
measured in empirical geometry \(\widehat\Sigma=X^\top X/n\), then
\(\widehat\Sigma P_0=0\) and this coupling would vanish.  If \(\Sigma=I\), then
\(P_i\Sigma P_0=P_iP_0=0\) as well.  Anti-shrinkage is therefore a population
prediction phenomenon caused by anisotropic geometry and favorable signal
alignment.

\noindent\textbf{Head--tail heuristic and endpoint shape.}
Proposition~\ref{prop:full-risk-oracle-geometry} allows arbitrary signal
geometry, and \(m_i\) may be nonmonotone; nevertheless the local quantities
suggest a robust head--tail intuition.  The variance price
\(v_i=(\sigma_\varepsilon^2/n)G_{ii}/\widehat\mu_i\) has an explicit
\(1/\widehat\mu_i\) factor, so lower empirical directions are intrinsically
more costly to push above ridgeless, while the missing-signal gain \(m_i\) is
generated by favorable row-space/null-space coupling and is most plausibly
concentrated on leading signal-rich directions in head-local regimes.  Thus
\(g_i=m_i-v_i\) is expected to become less favorable down the empirical
spectrum, and the oracle asks for a \emph{selective} leading correction.
This is precisely the shape comparison of
Section~\ref{sec:mechanism}: the endpoint pays lower-spectrum variance to
obtain head anti-shrinkage, whereas the finite-time prefix decouples the two
over a finite horizon.

\noindent\textbf{A two-coordinate example.}
The coupling is already visible in the smallest overparameterized example.
Take \(p=2\), \(n=1\), \(X=(1,0)\), and
$
\Sigma=
\begin{pmatrix}
1 & \rho\\
\rho & 1
\end{pmatrix},
\beta^\star=\beta_1 e_1+\beta_2 e_2 .
$
The empirical row space is \(\operatorname{span}(e_1)\), the null space is
\(\operatorname{span}(e_2)\), and \(\widehat\mu_1=1\).  Although these spaces
are Euclidean-orthogonal, population prediction couples them when
\(\rho\ne0\): \(P_1\Sigma P_0=\rho e_1e_2^\top\), so
\(m_1=\rho\beta_1\beta_2\), while the variance price from
Proposition~\ref{prop:full-risk-oracle-geometry} is
\(v_1=\sigma_\varepsilon^2\).
Thus increasing the fitted row-space coefficient above ridgeless is locally
risk-decreasing exactly when the covariance-aligned missing-signal gain
\(\rho\beta_1\beta_2\) exceeds this variance price.

\section{Gaussian head--tail models: effective ranks and adaptive signed paths}
\label{sec:gaussian-spike-models}

Both cases build on the Tsigler--Bartlett head--tail strategy, but signed
dynamics introduce an additional difficulty absent from positive ridge and
positive gradient flow.  Write \(K=XX^\top/n\) for the sample-space Gram
operator.  When the signed level \(\nu\) lies inside its empirical spectrum,
as it does in our critical constructions, the shifted generator
\(K-\nu I\) has negative eigenvalues on modes with
\(\widehat\mu_i<\nu\).  Consequently, the semigroup
\(e^{-t(K-\nu I)}\) is noncontractive on those modes, and the usual
positive-path contraction argument no longer controls perturbations over the
relevant time horizon.  We instead use Duhamel's semigroup formula to compare
the realized operator with a tractable head--tail floor model while explicitly
tracking its finite-time exposure.  In Case I, the transparent
common-spike-plus-flat model permits a global Duhamel bound at constant
rescaled time.  Case II keeps a heterogeneous head and allows a nonflat tail:
the tail trace sets the floor, the squared spectrum measures exposure, and
invariant-graph localization followed by within-head Duhamel control makes
off-diagonal head--tail coupling enter through a quadratic leave-and-return
term.

\noindent\textbf{Shared Gaussian framework.}
Throughout this section,
\[
\begin{gathered}
\displaystyle
X=Z\Sigma^{1/2},\qquad
\Sigma=\operatorname{diag}(\Lambda_H,\Lambda_T),\\
\Lambda_H=\operatorname{diag}(\lambda_1,\ldots,\lambda_k),\qquad
\Lambda_T=\operatorname{diag}(\lambda_{k+1},\ldots,\lambda_p),\\
\lambda_-=\min_{j\le k}\lambda_j,\quad
\lambda_+=\max_{j\le k}\lambda_j,\qquad
T_k=\sum_{j>k}\lambda_j,\quad
S_k=\sum_{j>k}\lambda_j^2,\quad
a=\frac{T_k}{n}.
\end{gathered}
\]
Here \(Z\in\R^{n\times p}\) has independent standard Gaussian entries,
\(\lambda_1\ge\cdots\ge\lambda_p>0\), and the resolved-head dimension
\(1\le k=k_n<n\) may grow with \(n\).  We assume
\(\lambda_+/\lambda_-=O(1)\) and \(a\asymp\lambda_-\).  The signal is
head-supported:
\[
\beta_j^\star=0\quad(j>k),
\qquad\text{equivalently}\qquad
\sum_{j>k}\lambda_j(\beta_j^\star)^2=0.
\]
We write
\(\Theta_j=\lambda_j(\beta_j^\star)^2\) and
\(\Theta_H=\sum_{j\le k}\Theta_j\), and assume that
\(\Theta_H,\sigma_\varepsilon^2\), and their reciprocals are bounded.
Write
\[
\chi_{n,k}:=\sqrt{\frac{k+\log n}{n}}
\]
for the Gaussian head-frame tolerance.
For the general-tail result define the Tsigler--Bartlett effective ranks and
the corresponding floor tolerance by
\[
\mathfrak r_k:=\frac{T_k}{\lambda_{k+1}},\qquad
\mathfrak R_k:=\frac{T_k^2}{S_k},\qquad
r_k:=\frac{n}{\mathfrak r_k}
+\sqrt{\frac{n}{\mathfrak R_k}},\qquad
w_k:=ar_k
=\lambda_{k+1}+\sqrt{\frac{S_k}{n}}.
\]
Thus \(a\) is the first-moment tail floor, whereas \(S_k\) controls both its
random width and the population-risk readout.  The condition \(r_k\to0\)
is the high-effective-rank regime in which \(K_T\) concentrates around
\(aI_n\).

We use one risk notation for both cases:
\[
\begin{aligned}
R_{\rm sign}(\nu,t)
&:=\Risk_{\rm pop}(f_{\nu,t}\mid X),
& R_{\rm scale}^\star
:=\inf_{c\in\R}
\Risk_{\rm pop}\!\left(
\mu\mapsto c\,\mathbf 1\{\mu>0\}\,\middle|\,X
\right),\\
R_{\rm neg}^\star
&:=
\inf_{0\le\nu<\widehat\mu_{\min}^+}
\Risk_{\rm pop}(A_\nu\mid X),
& R_{\rm GD+}^\star
:=
\inf_{t\ge0}\Risk_{\rm pop}(f_{0,t}\mid X),\\
R_{\rm ridge+}^\star
&:=
\inf_{\alpha\ge0}
\Risk_{\rm pop}\!\left(
\mu\mapsto\frac{\mu}{\mu+\alpha}\,\middle|\,X
\right),
& R_+^\star:=R_{\rm GD+}^\star\wedge R_{\rm ridge+}^\star.
\end{aligned}
\tag{Shared path and comparator risks}
\label{eq:shared-spike-comparator-risks}
\]

\begin{proposition}[Marchenko--Pastur bulk-width endpoint barrier]
\label{prop:journal-mp-barrier}
Under the shared framework, take
\(\Lambda_T=\lambda_T I_{d_T}\), write
\(\gamma_T=d_T/n\), and assume
\(a=\gamma_T\lambda_T\asymp\lambda_-\) and
\(\lambda_T/\lambda_-\to0\) and \(k=o(n)\).  Put
\(w_T:=\sqrt{a\lambda_T}\).  The sublinear-rank head does not change the
bulk-width order of the lower tail edge: there are constants \(0<c<C<\infty\)
such that, with probability tending to one,
\[
c w_T
\le
a-\widehat\mu_{\min}^+
\le
C w_T.
\]
Thus every admissible negative-ridge endpoint has
\(\nu<a-cw_T\) on the same event.  This order statement is the only
lower-edge input used in the separation theorems.  The classical
Marchenko--Pastur law \citep{marchenko1967distribution} predicts the sharper
displacement
\((2\sqrt{\gamma_T}-1)\lambda_T\sim2w_T\); the proposition deliberately does
not require that sharper constant.
\end{proposition}

\begin{figure}[t]
	\centering
	\includegraphics[width=\linewidth]{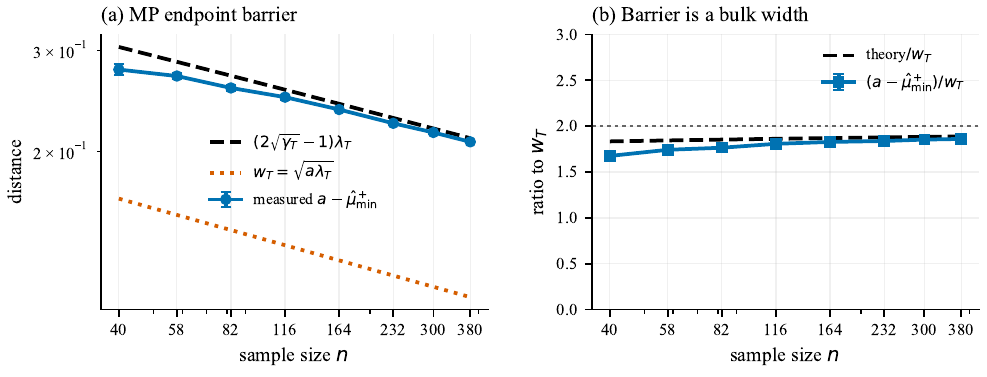}
	\caption{Numerical confirmation of the Marchenko--Pastur endpoint barrier on
	Gaussian common-spike-plus-flat designs (the validation experiment of
	\S\ref{sec:theory-aligned-simulations}, with \(a=1\) and
	\(\lambda_T\asymp0.1\,n^{-\zeta}\)).  Panel~(a): the measured gap
	\(a-\widehat\mu_{\min}^+\) between the tail floor and the smallest nonzero
	empirical eigenvalue follows the classical MP prediction
	\((2\sqrt{\gamma_T}-1)\lambda_T\) and runs parallel to
	\(w_T=\sqrt{a\lambda_T}\).  Panel~(b): the ratio
	\((a-\widehat\mu_{\min}^+)/w_T\) approaches the MP prediction \(2\).  The
	proposition and the separation theorem use only the resulting bulk-width
	order \(\asymp\sqrt{a\lambda_T}\), not this sharper constant.  Means and
	\(1.96\) standard-error bands over \(30\) designs.}
	\label{fig:mp-wall-sim}
\end{figure}

\subsection{Case I: Common-spike exact recovery}
\label{sec:mp-spike-flat-separation}

The common-spike specialization isolates the removable-pole mechanism.
Placing the signed level one head scale above the tail floor makes the head
generator neutral, so the stopped path recovers the head exactly while paying
only finite tail exposure.

\begin{assumption}[Common-spike specialization]
\label{ass:journal-spike-flat}
In the shared framework, set
\(k=O(1)\), \(\Lambda_H=\lambda_hI_k\), and
\(\Lambda_T=\lambda_TI_{d_T}\), with
\(\gamma_T=d_T/n\), \(a=\gamma_T\lambda_T\asymp\lambda_h\),
\(\lambda_T/\lambda_h\to0\), and
\(n\lambda_T/\lambda_h\to\infty\).  Then
\(\lambda_-=\lambda_+=\lambda_h\) and
\(\Theta_H=\lambda_h\|\beta_H^\star\|_2^2\); write
\(w_T=\sqrt{a\lambda_T}\).
\end{assumption}

For an integer \(m\ge1\), use the discrete response defined in
Section~\ref{sec:model} and set
\[
\begin{gathered}
R_{\rm sign}^{\rm disc}(\nu,\eta,m)
:=\Risk_{\rm pop}(f_{\nu,\eta,m}^{\rm disc}\mid X), \quad
(\nu_\star,t_\star,\eta_\star)
:=
(a+\lambda_h,\lambda_h^{-1},(m\lambda_h)^{-1}).
\end{gathered}
\]
In the ideal floor model, the continuous and discrete paths obey the four
identities
\[
\begin{aligned}
\lambda_hh_{\nu_\star,t_\star}(a+\lambda_h)
&=1,
&
h_{\nu_\star,t_\star}(a)
&=\frac{e-1}{\lambda_h},\\
\lambda_hh_{\nu_\star,\eta_\star,m}^{\rm disc}(a+\lambda_h)
&=1,
&
h_{\nu_\star,\eta_\star,m}^{\rm disc}(a)
&=\frac{(1+1/m)^m-1}{\lambda_h}
\le\frac{e-1}{\lambda_h}.
\end{aligned}
\]
Thus exact head recovery requires neither a long-time nor a floor-critical
limit, and its tail multiplier stays uniformly bounded.
For \(m=1\), the discrete filter is
\(h_{\nu,\eta,1}^{\rm disc}\equiv\eta\), so the signed level has no effect:
the construction reduces to a single large-step ordinary gradient update.  The
discrete upper bound below is therefore an implementation and endpoint
comparison, not a separation from unrestricted large-step ordinary GD.
Such a separation requires either a monotone-step comparator or an explicit
analysis of overshooting positive-GD polynomials.
Algorithm~\ref{alg:negative-shifted-gd} uses the stabilized subfamily: one
chooses \(m\) large enough that
\(\eta_\star(\widehat\mu_1^{\rm tr}+\nu_\star)\le c_\eta<1\).
The theorem retains all \(m\ge1\) because its discrete common-spike identity is
exact and is also useful for isolating the endpoint comparison.

\begin{theorem}[Common-spike total-risk separation]
\label{thm:journal-spike-flat-separation}
Under Assumption~\ref{ass:journal-spike-flat}, with probability tending to
one, the explicit continuous and discrete
supercritical paths satisfy, uniformly over integers \(m\ge1\),
\[
\begin{aligned}
R_{\rm sign}(\nu_\star,t_\star)
\vee
R_{\rm sign}^{\rm disc}(\nu_\star,\eta_\star,m)
&\lesssim
\frac{\sigma_\varepsilon^2k}{n}
+
(\Theta_H+\sigma_\varepsilon^2)
\frac{a\lambda_T}{\lambda_h^2},\\
R_{\rm neg}^\star
&\gtrsim
\frac{\sigma_\varepsilon^2k}{n}
+
\frac{\sigma_\varepsilon\sqrt{\Theta_Ha\lambda_T}}{\lambda_h},
\quad 
R_+^\star
\gtrsim
\Theta_H\left(\frac{a}{\lambda_h+a}\right)^2.
\end{aligned}
\]
In the interior regime
\(a,\lambda_h,\Theta_H,\sigma_\varepsilon^2\asymp1\),
\[
\frac{R_{\rm sign}(\nu_\star,t_\star)}{R_{\rm neg}^\star}
\lesssim
\sqrt{\lambda_T/\lambda_h}\longrightarrow0,
\qquad
\frac{R_{\rm sign}(\nu_\star,t_\star)}
{R_+^\star}
\lesssim
\lambda_T/\lambda_h\longrightarrow0.
\]
\end{theorem}

\noindent\textbf{A polynomial separation example.}
Fix \(0<\zeta<1\), take \(k=O(1)\),
\[
\lambda_h\asymp a\asymp\Theta_H\asymp\sigma_\varepsilon^2\asymp1,
\qquad
\lambda_T=n^{-\zeta},
\qquad
d_T\asymp n^{1+\zeta}.
\]
Then \(a=(d_T/n)\lambda_T\asymp1\) and
\(n\lambda_T\to\infty\), so Assumption~\ref{ass:journal-spike-flat} holds.
Theorem~\ref{thm:journal-spike-flat-separation} yields, uniformly in
\(m\ge1\),
\[
\begin{aligned}
R_{\rm sign}(\nu_\star,t_\star)
\vee R_{\rm sign}^{\rm disc}(\nu_\star,\eta_\star,m)
=O_{\mathbb P}(n^{-\zeta}),\quad 
R_{\rm neg}^\star=\Omega_{\mathbb P}(n^{-\zeta/2}), \quad
R_+^\star=\Omega_{\mathbb P}(1).
\end{aligned}
\]
Thus the signed path improves on the best admissible negative-ridge endpoint
by the factor \(n^{-\zeta/2}\) and on positive shrinkage by
\(n^{-\zeta}\).

The theorem compares explicit stopped continuous and discrete paths with
oracle-tuned endpoints on the same random-design model.  Positive shrinkage cannot undo the
floor-induced head attenuation.  Admissible negative ridge can move above
ridgeless, but its pole must remain below the MP edge and its optimized
bias--variance tradeoff remains of square-root order.  The stopped path instead
places its removable pole at the common head location \(a+\lambda_h\), exactly
recovers the head at time \(1/\lambda_h\), and pays only order
\(a\lambda_T/\lambda_h^2\) for tail exposure and random-design transfer.
Proofs are given in Appendix~\ref{app:journal-spike-flat-proofs}.

This separation is deliberately transparent rather than maximally general.
The exact equality of the head spikes isolates the removable-pole mechanism,
but it also permits a uniformly rescaled ridgeless estimator to reproduce the
same single-scale correction.  Section~\ref{sec:multispike-floor-critical}
removes that degeneracy: heterogeneous spikes require a nonconstant head
correction, and a floor-critical path achieves it through localized
finite-exposure dynamics.

\subsection{Case II: Heterogeneous heads and high-effective-rank tails}
\label{sec:multispike-floor-critical}

The common-spike construction isolates the removable-pole mechanism, but its
flat tail and single head scale hide two distinct issues.  A nonflat tail
creates a random floor whose location is controlled by \(T_k\) and whose
prediction-risk exposure is controlled by \(S_k\).  A heterogeneous head
also rules out repairing the floor by uniformly rescaling ridgeless
regression.  This subsection treats both effects in one Tsigler--Bartlett
head--tail theorem; the flat-tail model is then a corollary.

\begin{assumption}[Heterogeneous-head, general-tail specialization]
\label{ass:journal-multispike}
Use the shared Gaussian framework with diagonal head
\(\Lambda_H=\operatorname{diag}(\lambda_1,\ldots,\lambda_k)\), where the
\(\lambda_j\) need not be equal.  Assume
\[
r_k\to0,
\qquad
\frac{k w_k^2}{S_k}\to0.
\]
The second condition is the projected tail-square-mass requirement for a
growing head; it implies both \(k=o(n)\) and
\(S_k/\lambda_{k+1}^2\to\infty\).
\end{assumption}

For a uniformly rescaled ridgeless estimator, the ideal floor-model recovery
at spike \(j\) is \(cq_j\).  Its best scalar approximation is
\[
\begin{aligned}
q_j&:=\frac{\lambda_j}{\lambda_j+a},
&
\Delta_{\rm shape}
&:=\inf_{c\in\R}\sum_{j=1}^k\Theta_j(1-cq_j)^2,\\
\Delta_{\rm shape}
&=\frac{\Theta_1\Theta_2(q_1-q_2)^2}
{\Theta_1q_1^2+\Theta_2q_2^2}
&&\text{when \(k=2\).}
\end{aligned}
\]
Thus \(\Delta_{\rm shape}\asymp\Theta_H\) when two separated head scales
carry nonvanishing prediction energy.

The floor-critical tuning balances head recovery against the squared-spectrum
tail exposure:
\[
\begin{gathered}
V_{T,k}:=\frac{\sigma_\varepsilon^2S_k}{n\lambda_-^2},
\qquad
L_k:=\frac12\log\!\left(1+\frac{\Theta_H}{V_{T,k}}\right),\\
\nu_{\rm fc}:=a,\qquad
t_{\rm fc}:=\frac{L_k}{\lambda_-},\qquad
\frac1nX_H^\top h_{a,t_{\rm fc}}(K_H+aI_n)X_H
=I_k-e^{-t_{\rm fc}G_H},
\end{gathered}
\]
where \(G_H=X_H^\top X_H/n\).  The trace \(T_k\) therefore determines the
theoretically chosen signed level, while \(S_k\) determines how long the path
can safely expose the tail.  All path and comparator risks use the shared
notation in \eqref{eq:shared-spike-comparator-risks}.

\begin{theorem}[Trace--square-mass separation with heterogeneous tails]
\label{thm:journal-multispike-floor-critical}
Under Assumption~\ref{ass:journal-multispike}, suppose
\(r_kL_k\to0\) and \(L_k\chi_{n,k}\to0\).
Then, with probability tending to one,
\[
\begin{aligned}
R_{\rm sign}(a,t_{\rm fc})
&\lesssim
\frac{\sigma_\varepsilon^2k}{n}
+
V_{T,k}(1+L_k^2)
+
\Theta_H\frac{w_k^2}{\lambda_-^2},\\
R_{\rm scale}^\star
&\ge\Delta_{\rm shape}-o(1)\Theta_H,
\quad 
R_+^\star
\gtrsim
\sum_{j=1}^k
\Theta_j\left(\frac{a}{a+\lambda_j}\right)^2 .
\end{aligned}
\tag{General-tail risk bounds}
\label{eq:multispike-risk-bounds}
\]
Moreover, for a numerical constant \(C_0\), every admissible negative-ridge
endpoint satisfies the trace--square-mass lower envelope
\[
R_{\rm neg}^\star
\gtrsim
\frac{\sigma_\varepsilon^2k}{n}
+
\inf_{b\ge0}
\left\{
\Theta_H\frac{b^2}{(b+\lambda_+)^2}
+
\frac{\sigma_\varepsilon^2S_k}
{n(b+C_0w_k)^2}
\right\}.
\tag{General-tail endpoint envelope}
\label{eq:multispike-endpoint-envelope}
\]
If
\(
b_k^\star:=
\left(
\frac{\sigma_\varepsilon^2S_k\lambda_+^2}
{n\Theta_H}
\right)^{1/4},
w_k=o(b_k^\star),
\)
then the envelope simplifies to
\[
R_{\rm neg}^\star
\gtrsim
\frac{\sigma_\varepsilon^2k}{n}
+
\frac{\sigma_\varepsilon\sqrt{\Theta_HS_k/n}}{\lambda_+}.
\tag{Square-mass endpoint lower bound}
\label{eq:multispike-endpoint-simplified}
\]
\end{theorem}

Here \(k=k_n\) may grow; its two explicit costs are the resolved-head variance
\(\sigma_\varepsilon^2k/n\) and the head-frame tolerance \(\chi_{n,k}\).
In particular, \(R_{\rm scale}^\star\ge\Delta_{\rm shape}/2\) whenever
\(\Delta_{\rm shape}\gtrsim\Theta_H\).
The theorem restores the Tsigler--Bartlett separation of roles.  The first
effective rank \(\mathfrak r_k\) and the trace \(T_k\) control whether the
tail creates a stable sample-space floor.  The second effective rank
\(\mathfrak R_k\) and square mass \(S_k\) control tail noise, head-to-tail
leakage, and the negative-endpoint variance readout.  The new ingredient is
the noncontractive transfer: an invariant graph aligns the realized head and
tail clusters, after which a within-head Duhamel integral charges direct
compression linearly and off-diagonal leave--return coupling quadratically.
This removes the spurious \(L_k^2\) multiplier from the operator-width term;
the surviving \(V_{T,k}L_k^2\) term is the genuine variance exposure of modes
near the floor.
The assumptions are population spectral conditions, not an imported
realized-event package; the trace-floor, localized-tail, and projected
square-mass events are derived from the full Gaussian design in
Lemma~\ref{lem:journal-multispike-gaussian-event}.

\begin{corollary}[Heterogeneous spikes with a flat tail]
\label{cor:journal-multispike-flat-tail}
In Theorem~\ref{thm:journal-multispike-floor-critical}, set
\(\Lambda_T=\lambda_TI_{d_T}\),
\(\gamma_T=d_T/n\), and \(a=\gamma_T\lambda_T\).  If
\(\lambda_T/\lambda_-\to0\), then
\[
\frac{S_k}{n}=a\lambda_T,\qquad
w_k\asymp\sqrt{a\lambda_T},\qquad
L_k=\frac12\log\!\left(
1+\frac{\Theta_H\lambda_-^2}
{\sigma_\varepsilon^2a\lambda_T}
\right).
\]
Under
\[
\frac{w_k}{\lambda_-}L_k\longrightarrow0,
\qquad
L_k\chi_{n,k}\longrightarrow0,
\tag{Flat-tail exposure and head-frame sharpness}
\label{eq:multispike-flat-tail-sharpness}
\]
the floor-critical path satisfies
\[
\begin{aligned}
R_{\rm sign}(a,t_{\rm fc})
\lesssim
\frac{\sigma_\varepsilon^2k}{n}
+
\left\{
\Theta_H+\sigma_\varepsilon^2(1+L_k^2)
\right\}
\frac{a\lambda_T}{\lambda_-^2}, \quad
R_{\rm neg}^\star
\gtrsim
\frac{\sigma_\varepsilon^2k}{n}
+
\frac{\sigma_\varepsilon\sqrt{\Theta_Ha\lambda_T}}{\lambda_+}.
\end{aligned}
\]
Proposition~\ref{prop:journal-mp-barrier} additionally identifies the
endpoint wall \(a-\widehat\mu_{\min}^+\asymp_{\mathbb P}
\sqrt{a\lambda_T}\).  If
\(\Delta_{\rm shape}\gtrsim\Theta_H\) and, in addition,
\[
\frac{w_k}{\lambda_-}L_k^2\longrightarrow0,
\qquad
\frac{k\lambda_-}{nw_k}\longrightarrow0,
\tag{Flat-tail total-risk separation}
\label{eq:multispike-flat-tail-total-separation}
\]
then
\[
\frac{R_{\rm sign}(a,t_{\rm fc})}{R_{\rm neg}^\star}
\lesssim
\frac{w_k}{\lambda_-}(1+L_k^2)
+
\frac{k\lambda_-}{nw_k}
\to0,
\qquad
\frac{R_{\rm sign}(a,t_{\rm fc})}
{R_{\rm scale}^\star\wedge R_+^\star}
\lesssim
\frac{k}{n}
+
\frac{\lambda_T}{\lambda_-}(1+L_k^2)
\to0 .
\]
\end{corollary}

\begin{corollary}[A truncated power-law tail]
\label{cor:journal-multispike-power-law}
Fix \(0<\alpha<1/2\), take \(k=O(1)\), let
\(d_n=\lfloor n^q\rfloor\), and set
\[
\lambda_{k+\ell}
=c_n\ell^{-\alpha},
\qquad
c_n:=\frac{na}{\sum_{\ell=1}^{d_n}\ell^{-\alpha}},
\qquad
1\le\ell\le d_n,
\]
so that \(T_k=na\).  If
\[
\max\left\{
\frac1{1-\alpha},
\frac3{3-4\alpha}
\right\}
<q<3,
\tag{Power-law dimension condition}
\label{eq:multispike-power-law-dimension}
\]
then
\[
\frac{S_k}{n}\asymp a^2\frac{n}{d_n},
\qquad
\frac{w_k}{a}
\asymp
\frac{n}{d_n^{1-\alpha}}+\sqrt{\frac{n}{d_n}},
\qquad
L_k\asymp\log(d_n/n).
\tag{Power-law TB scales}
\label{eq:multispike-power-law-scales}
\]
Consequently,
\[
R_{\rm sign}(a,t_{\rm fc})
\lesssim
\frac{\sigma_\varepsilon^2k}{n}
+
\Theta_H\frac{n^2}{d_n^{2(1-\alpha)}}
+
\left\{
\Theta_H+\sigma_\varepsilon^2
\bigl[1+\log^2(d_n/n)\bigr]
\right\}\frac{n}{d_n},
\tag{Power-law signed bound}
\label{eq:multispike-power-law-signed}
\]
while the negative-endpoint term in
\eqref{eq:multispike-endpoint-simplified} is of order
\(\sqrt{n/d_n}\).  Thus the signed path separates from the negative endpoint,
and also from uniform rescaling and positive shrinkage whenever
\(\Delta_{\rm shape}\gtrsim\Theta_H\).
\end{corollary}

\begin{corollary}[Discrete floor-critical GD]
\label{cor:journal-multispike-discrete}
Under Theorem~\ref{thm:journal-multispike-floor-critical}, let
\(\eta_m=t_{\rm fc}/m\).  If
\(m\gtrsim(\lambda_-/w_k)(1+L_k^2)\), then the \(m\)-step negative-shifted GD
estimator with signed level \(\nu=a\) satisfies the signed upper bound in
\eqref{eq:multispike-risk-bounds}.  On the theorem event this iteration
condition also implies
\(\eta_m(\widehat\mu_1+a)=o(1)\), so the discrete path belongs to the
small-step family used in Algorithm~\ref{alg:negative-shifted-gd}.
\end{corollary}

The iteration requirement keeps the Euler head-map error below the sharpened
continuous-time trace--square-mass envelope; it is not an
optimization-complexity claim.  Proofs are in
Appendix~\ref{app:journal-multispike-floor-critical-proofs}.

\subsection{From separation theorems to the deployable estimator}
\label{sec:validation-oracle}

The separations above are stated for specific signed paths.  The following
finite-grid hold-out oracle inequality shows that the deployable
Algorithm~\ref{alg:negative-shifted-gd}, which fixes a stable
training-measurable step-size rule and selects \((\nu,m)\) by validation,
inherits any such separation up to an explicit validation-size penalty.

\begin{theorem}[Finite-grid hold-out oracle for Algorithm~\ref{alg:negative-shifted-gd}]
\label{thm:finite-grid-validation-oracle}
Let \(\mathcal D_{\rm tr}\) denote the entire training sample and let
\(\{\widehat\beta_g:g\in\mathcal G\}\) be any nonempty finite collection of
\(\mathcal D_{\rm tr}\)-measurable coefficient vectors, where
\(\mathcal G=\mathcal V\times\mathcal T\) for Algorithm
\(\ref{alg:negative-shifted-gd}\).  Suppose the independent validation sample
has size \(n_{\rm val}\).  Conditional on the training sample, the validation
pairs are i.i.d. and obey
\[
x_i^{\rm val}\sim N(0,\Sigma),
\qquad
y_i^{\rm val}=x_i^{{\rm val}\top}\beta^\star
+\varepsilon_i^{\rm val},
\qquad
\varepsilon_i^{\rm val}\sim N(0,\sigma_\varepsilon^2).
\]
For every \(i\), \(\varepsilon_i^{\rm val}\) is independent of
\(x_i^{\rm val}\), and the entire validation sample is independent of
\(\mathcal D_{\rm tr}\).
Assume \(\sigma_\varepsilon^2>0\).
For \(0<\delta<1\), put
\[
z_{\rm val}:=\log\frac{2|\mathcal G|}{\delta},
\qquad
\epsilon_{\rm val}
:=
2\sqrt{\frac{z_{\rm val}}{n_{\rm val}}}
+\frac{2z_{\rm val}}{n_{\rm val}},
\]
and assume \(\epsilon_{\rm val}<1\).  If \(\widehat g\) minimizes validation
squared error over \(\mathcal G\), then, conditional on
\(\mathcal D_{\rm tr}\), with probability at least \(1-\delta\),
\[
\begin{aligned}
\sigma_\varepsilon^2+
\|\widehat\beta_{\widehat g}-\beta^\star\|_\Sigma^2
&\le
\frac{1+\epsilon_{\rm val}}{1-\epsilon_{\rm val}}
\min_{g\in\mathcal G}
\left\{
\sigma_\varepsilon^2+
\|\widehat\beta_g-\beta^\star\|_\Sigma^2
\right\},\\
\|\widehat\beta_{\widehat g}-\beta^\star\|_\Sigma^2
&\le
\frac{1+\epsilon_{\rm val}}{1-\epsilon_{\rm val}}
\min_{g\in\mathcal G}
\|\widehat\beta_g-\beta^\star\|_\Sigma^2
+
\frac{2\epsilon_{\rm val}}{1-\epsilon_{\rm val}}
\sigma_\varepsilon^2 .
\end{aligned}
\tag{Validation oracle inequality}
\label{eq:finite-grid-validation-oracle}
\]
Consequently, if \(\delta=\delta_n\to0\),
\(\epsilon_{\rm val}=o(1)\), a sequence of grids contains a theoretical
candidate \(g_0\) whose design-conditional mean excess risk satisfies
\[
\mathbb E_{\varepsilon_{\rm tr}}
\!\left[
\|\widehat\beta_{g_0}-\beta^\star\|_\Sigma^2
\mid X_{\rm tr}
\right]
=o_{\mathbb P}(B_n),
\]
and
\(\epsilon_{\rm val}\sigma_\varepsilon^2=o(B_n)\), then the
validation-selected excess risk is \(o_{\mathbb P}(B_n)\).  Thus any
separation theorem with comparator lower scale \(B_n\) transfers to
Algorithm~\ref{alg:negative-shifted-gd} under this explicit validation-size
condition.
\end{theorem}

\noindent\textbf{Full-grid scope.}
Theorem~\ref{thm:finite-grid-validation-oracle} compares with the entire
predeclared candidate grid.  Patience stopping may be used as a computational
shortcut, but then the same argument gives an oracle comparison only over the
candidates actually evaluated; transfer of a separation theorem requires that
its certified candidate be among them.  The theory-aligned hold-out experiment
below evaluates its full finite \((\nu,t)\) grid, while the larger aggregate
study uses patience only as an implementation heuristic.

\noindent\textbf{Validation size.}
When \(\sigma_\varepsilon^2\asymp1\),
\(\epsilon_{\rm val}\asymp
\sqrt{\log(|\mathcal G|/\delta)/n_{\rm val}}\) in the usual regime.
Preserving an \(O(r_n)\) candidate upper rate therefore requires
\(n_{\rm val}\gg r_n^{-2}\log(|\mathcal G|/\delta)\).  In the polynomial
common-spike example, with
\(s_n=\lambda_T/\lambda_h=n^{-\zeta}\), retaining the full \(O(s_n)\)
selected-risk rate requires
\(n_{\rm val}\gg n^{2\zeta}\log(|\mathcal G|/\delta)\); retaining only the
separation from the \(\Omega(s_n^{1/2})\) negative-endpoint scale requires
\(n_{\rm val}\gg n^\zeta\log(|\mathcal G|/\delta)\).

\section{Experiments}
\label{sec:experiments}

\subsection{Simulation settings}
\label{sec:simulation-settings}

Each experiment draws \(X=Z\Sigma^{1/2}\), where \(Z\) has iid standard
Gaussian entries,
\(\Sigma^{1/2}
=Q\operatorname{diag}(\sqrt{\Lambda_1},\ldots,\sqrt{\Lambda_p})Q^\top\),
and \(Q\) is a random orthogonal matrix.  The aggregate grid uses
\(n=100\), \(n_{\rm val}=n_{\rm test}=500\), \(p\in\{500,1000\}\), and
\(n_{\rm reps}=50\),
with \(\sigma_\varepsilon=0.3\), \(k_{\rm spike}=5\), and
\(\|\beta^\star\|_2^2=5\).  The signal--covariance pairs split into a
theory-aligned cell and deliberate departures from it.
For a fitted estimator \(\widehat\beta\), the aggregate table reports
\[
\operatorname{RMSE}_{\rm test}
:=
\left\{
\frac1{n_{\rm test}}
\sum_{\ell=1}^{n_{\rm test}}
\bigl(y_\ell^{\rm test}
-x_\ell^{{\rm test}\top}\widehat\beta\bigr)^2
\right\}^{1/2},
\]
which includes irreducible test noise.

\noindent\textbf{Theory-aligned setting.}
The \emph{spectral-sparse} signal places five iid \(N(0,1)\) coefficients on
the top five population PCs (sorted decreasingly, all other coordinates zero,
then rescaled).  Combined with the \emph{Spiked+Flat} covariance---spikes
\((1,4/5,3/5,2/5,1/5)\) over a flat tail of level \(0.05\)---the supported
directions are exactly the spikes, so the signal is head-supported and the
spectrum is gapped.  This is the shared Gaussian framework of
Section~\ref{sec:gaussian-spike-models} with unequal head spikes,
instantiating the heterogeneous flat-tail
Corollary~\ref{cor:journal-multispike-flat-tail} rather than the equal-spike
special case.

\noindent\textbf{Robustness beyond theory.}
The remaining settings each break one framework assumption.  The
\emph{source-condition} signal draws \(g_j\stackrel{\rm iid}{\sim}N(0,1)\) and
sets \(\beta_j^\star=(\Lambda_j/\Lambda_1)g_j\), followed by the same norm
rescaling; its energy decays down the entire spectrum, so the signal is not
head-supported even over a gapped covariance.  The \emph{power-law} spectra
\(\Lambda_i=i^{-\alpha}\), \(\alpha\in\{0.5,0.75,1\}\), are \emph{no-gap}:
head and tail merge into a single continuous ladder, outside the gapped
regime, even when paired with the head-supported spectral-sparse signal.  Any
row with either departure is a robustness test rather than a certified
instance of the theorems.

Ridge SVD and ridge GD are selected by validation separately over
nonnegative and nonpositive parameter grids:
\[
\Lambda_+
=
\{0\}\cup\operatorname{logspace}(-5,2,20),
\qquad
\Lambda_-
=
\{0\}\cup\{-x:x\in\operatorname{logspace}(-5,2,20)\}.
\]
For stable signed SVD, a candidate is retained only if every denominator
\(\widehat\mu_i+\lambda\) is positive and remains at least
\(10^{-8}\max\{1,\widehat\mu_1\}\) from a spectral pole.  Signed GD uses the
same \(\Lambda_-\) grid as a finite-time path class, with exposure controlled
by validation and early stopping.  All GD runs use full-batch updates, learning
 rate \(10^{-2}\), at most \(10^4\) epochs, validation after every epoch,
 patience \(500\), and no gradient clipping.  Reported filters and RMSEs come
 from the actual discrete iterates; the continuous NS-GF filter is used only as
 the analytic small-step envelope.  As noted above, patience is a computational
 shortcut in this aggregate study rather than part of the full-grid transfer
 theorem.

The code records the seed and method-level RMSE for every replication.
Paired rows are used to report signed-SVD minus NS-GD and
nonnegative-GD minus NS-GD differences with paired standard errors.
The full grid and diagnostic figures were generated on a standard CPU
workstation; a complete rerun takes a few hours, depending on processor count
and BLAS threading.

\subsection{Theory-aligned separation curves}
\label{sec:theory-aligned-simulations}

To isolate the rate mechanism in Theorem~\ref{thm:journal-spike-flat-separation}
and Corollary~\ref{cor:journal-multispike-flat-tail}, we run a second paired
experiment in the theorem's interior regime.  We take eight sample sizes
\[
n\in\{40,58,82,116,164,232,300,380\},\qquad
\zeta=0.35,\qquad
d_T=\left\lceil 10n^{1+\zeta}\right\rceil,\qquad
\lambda_T=\frac{n}{d_T},
\]
so \(a=d_T\lambda_T/n=1\) and
\(\lambda_T\asymp0.1n^{-\zeta}\).  We use
\(\sigma_\varepsilon=0.5\) and \(30\) paired replications.  The common-spike
model has three unit spikes with
\(\Theta_j=1/3\); the heterogeneous model has spikes
\((1.5,0.9,0.45)\) with \(\Theta_j=1\).
Because the validation selection below must predict on held-out points, the
flat-tail block is materialized as explicit Gaussian columns from the exact
design law \(X_T=\sqrt{\lambda_T}\,Z_T\); the moderate decay power \(\zeta\)
keeps the tail dimension small while preserving both overparameterization and
the unit floor.

For every realized design, we evaluate conditional population prediction risk
exactly, analytically averaging over the training noise.  All comparator
classes---negative-ridge endpoint, positive ridge/ordinary gradient flow, and
uniform rescaling---are tuned on dense finite grids using the exact conditional
risk; we call these \emph{dense-grid oracles}.  We
report the negative-shifted gradient \emph{flow} (NS-GF)---the small-step limit
of Algorithm~\ref{alg:negative-shifted-gd}, whose continuous filter is the
analytic envelope used throughout---two ways over the same finite \((\nu,t)\)
grid: a \emph{dense-grid oracle} NS-GF that minimizes the exact conditional
risk, giving a like-for-like comparison against the other dense-grid
comparators, and a
\emph{validation-selected} NS-GF that instead minimizes hold-out mean-squared
error on an independent sample of size \(n_{\rm val}=400\) and never uses
\(S_k\), \(\Theta_H\), or \(\sigma_\varepsilon\).  For the validation path we
report the \emph{realized} excess risk of the selected estimator; it tracks its
own oracle closely, so the practically-tuned path essentially attains the oracle
separation.

\begin{figure}[t]
	\centering
	\includegraphics[width=0.8\linewidth]{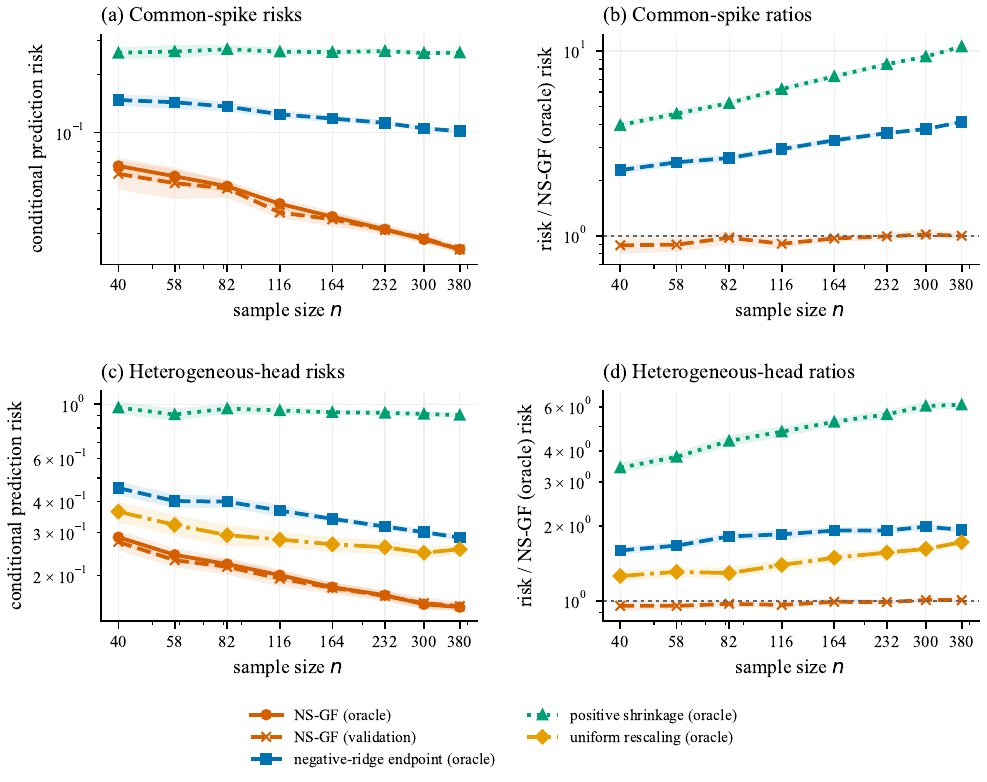}
	\caption{
		Conditional-risk separation in exact Gaussian spike-plus-flat designs:
		common spikes in (a)--(b), heterogeneous spikes in (c)--(d); left
		column risk, right column paired ratios to the NS-GF dense-grid oracle
		(horizontal line: equality).  Means and \(1.96\) standard-error bands
		over \(30\) paired replications; tuning protocol as described in the
		text.  Comparator ratios grow with \(n\), while the validation-selected
		path nearly coincides with its dense-grid oracle.
	}
	\label{fig:theory-separation-simulations}
\end{figure}

The separation is visible throughout the displayed range and strengthens with
\(n\).  At \(n=380\), the mean ratios to the NS-GF dense-grid oracle are \(4.1\)
for the best
admissible negative-ridge endpoint and \(10.6\) for positive shrinkage in the
common-spike model, and \(1.9\), \(6.1\), and \(1.7\) for the endpoint, positive
shrinkage, and best uniform rescaling in the heterogeneous model.  The
validation-selected curve reports the \emph{realized} excess risk of the
selected continuous-flow estimator, so selection and evaluation share the same
training noise; it tracks the dense-grid NS-GF oracle closely---the two curves
nearly coincide in panels~(a) and~(c)---and can sit marginally below the
fixed-filter oracle
because validation adapts its choice to each realization.  The
validation-tuned path therefore essentially attains the dense-grid separation.
The fitted log-risk
slopes track the predicted orders: in the common-spike panel the oracle endpoint
decays at rate \(n^{-0.18}\) and positive shrinkage is essentially flat,
matching the \(n^{-\zeta/2}\) and \(n^{0}\) predictions, while NS-GF decays at
least as fast as its \(n^{-\zeta}\) upper bound.  These finite-sample curves illustrate the
theorem's separation ordering; they are not used to assert sharp empirical
exponents.

The same designs also reproduce the Marchenko--Pastur wall of
Figure~\ref{fig:mp-wall-sim} (Proposition~\ref{prop:journal-mp-barrier}): the
admissible endpoint's pole sits a bulk width \(\asymp\sqrt{a\lambda_T}\) below
the floor, the geometry that forces the \(n^{-\zeta/2}\) endpoint rate in
panel~(b) of Figure~\ref{fig:theory-separation-simulations}.

\begin{figure}[t]
	\centering
	\includegraphics[width=0.8\linewidth]{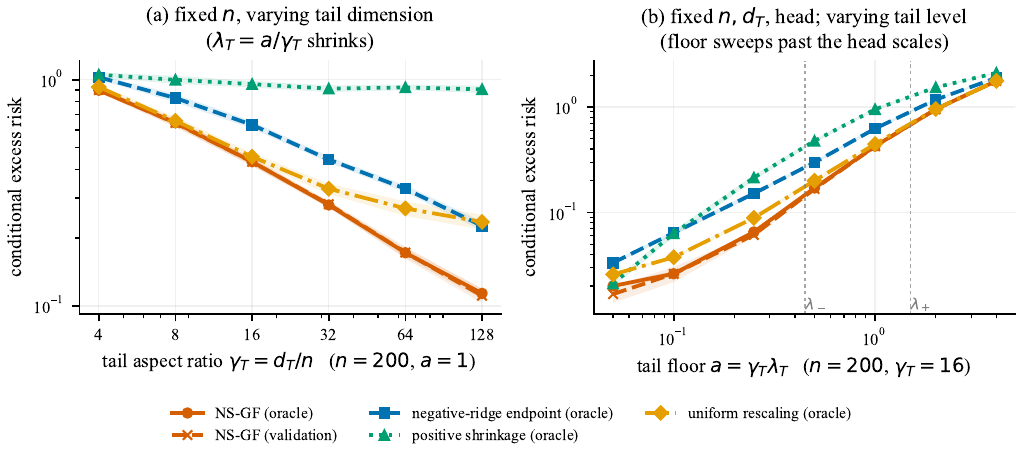}
	\caption{Separation across tail geometry at fixed sample size
	(\(n=200\), heterogeneous head, \(20\) paired replications, \(1.96\)
	standard-error bands; tuning as in
	Figure~\ref{fig:theory-separation-simulations}).
	(a)~Fixed floor \(a=1\), growing tail dimension \(d_T=\gamma_Tn\): the
	separation widens with overparameterization.
	(b)~Fixed \(n\), \(\gamma_T=16\), and head, with the tail level sweeping
	the floor past the head scales (\(\lambda_-\), \(\lambda_+\) marked): the
	separation peaks moderately below \(\lambda_-\) and closes at both
	extremes.}
	\label{fig:separation-dims}
\end{figure}

Figure~\ref{fig:separation-dims} completes this diagnostic by varying the tail
geometry at a fixed sample size \(n=200\), with the same heterogeneous head and
oracle/validation protocol over \(20\) paired replications.  Panel~(a) fixes
the floor \(a=1\) and grows the tail aspect ratio \(\gamma_T=d_T/n\) from
\(4\) to \(128\), so the tail level \(\lambda_T=a/\gamma_T\) shrinks.  The
separation widens monotonically with overparameterization: the oracle
endpoint-to-NS-GF risk ratio grows from \(1.1\) to \(2.0\), the
positive-shrinkage ratio from \(1.2\) to \(8.0\), and uniform rescaling loses
its near-parity (\(1.0\) to \(2.1\)).  This is the theorem geometry at fixed
floor: positive shrinkage stays pinned near its head-bias floor while the
NS-GF risk falls with \(\lambda_T\), and the admissible endpoint remains
limited by the wall retreat \(\asymp\sqrt{a\lambda_T}\), which shrinks more
slowly than the path's tail exposure.

Panel~(b) instead fixes \(n\), the tail dimension (\(\gamma_T=16\)), and the
head, and sweeps the tail level so the floor \(a=\gamma_T\lambda_T\) moves from
far below to above the head scales.  The separation is non-monotone with a
matched-floor peak: the ratios are largest when the floor sits moderately
below the smallest head eigenvalue (endpoint ratio \(2.5\) near \(a=0.1\),
positive-shrinkage ratio \(3.3\) near \(a=0.25\), with \(\lambda_-=0.45\)),
and they close at both extremes.  When \(a\) is very small there is almost no
implicit shrinkage to correct, so positive shrinkage is nearly optimal (ratio
\(1.05\)) while the endpoint still pays a residual pole-exposure cost
(\(1.66\)); once \(a\) rises past \(\lambda_+\), the floor submerges the head
and no admissible filter shape helps (all ratios near one).  The signed
path's gains therefore concentrate precisely where the shared framework
places them: an implicit floor matched to, or moderately below, the
signal-bearing head scales.

\subsection{Aggregate results and spectral diagnostics}

Table~\ref{tab:rmse-deterministic} summarizes the simulation results for the signed endpoint-versus-path comparison, with paired within-replication RMSE differences in the last two columns.  Each row averages \(50\) paired
replications from Gaussian designs with \(n=100\) and
\(p\in\{500,1000\}\).  We compare validation-selected nonnegative ridge SVD,
signed ridge SVD, nonnegative ridge GD, and NS-GD, each tuned within its
own admissible class.  Following the split of
Section~\ref{sec:simulation-settings}, the table's theory-aligned cell is the
spectral-sparse signal with the Spiked+Flat covariance (the rows marked
\(^\dagger\)): a head-supported signal over a gapped spectrum with unequal
head spikes, instantiating the heterogeneous
Corollary~\ref{cor:journal-multispike-flat-tail} rather than the equal-spike
special case.  Every other row breaks one framework assumption: the
source-condition signal is not head-supported, and the power-law rows use
\emph{no-gap} spectra, in which head and tail merge into a single continuous
ladder outside the gapped regime of
Section~\ref{sec:gaussian-spike-models}.  (The truncated power-law
Corollary~\ref{cor:journal-multispike-power-law} allows a heterogeneous
power-law tail, but its dimension condition implies
\(\lambda_{k+1}/a\to0\), so a finite head remains separated from that tail.)
Those rows, together with the power-law diagnostic
Figure~\ref{fig:finalfilter-powerlaw-one}, are \emph{robustness tests}
rather than certified instances of the theorems.  The exact
full-risk result is conditional on the observed design, while the two
Gaussian theorems in Section~\ref{sec:gaussian-spike-models} provide
regime-specific random-design guarantees. Across the \(16\) settings in Table~\ref{tab:rmse-deterministic},
early-stopped NS-GD is best in every row, in the theory-aligned cell and in
every robustness row alike.  The paired columns show that this
endpoint--path gap is positive and larger than twice its paired standard error
in every setting.  The second paired column also shows that NS-GD improves
over nonnegative GD in every row, with an average gain of about \(0.082\) RMSE.
The exact conditional result and the two Gaussian case proofs also suggest why
the signed-path advantage can change with overparameterization.
Changing \(p/n\) changes both the empirical null space and the lower-edge scale.
The theory therefore predicts spectrum-dependent behavior instead of a simple
monotone effect of \(p\): the endpoint wall is governed by the
Marchenko--Pastur lower edge in Proposition~\ref{prop:journal-mp-barrier},
while the full-oracle correction depends
on the realized missing-signal coupling in
Proposition~\ref{prop:full-risk-oracle-geometry}.

\begin{table}[t]
	\centering
	\caption{Test RMSE: mean\,(SE) over $50$ paired replications. Bold RMSE entries mark the best method(s) per row. Paired $\Delta$ columns show method\,1 $-$ method\,2 RMSE. Rows marked $^\dagger$ (spectral-sparse signal with Spiked+Flat covariance) form the theory-aligned setting of Section~\ref{sec:simulation-settings}; all other rows are robustness tests with a source-condition signal and/or a no-gap power-law spectrum.}
	\label{tab:rmse-deterministic}
	\small
	\setlength{\tabcolsep}{3.5pt}
	\resizebox{\textwidth}{!}{%
	\begin{tabular}{lll cc cc cccc @{\hspace{4pt}} cc}
		\toprule
		&   &   & \multicolumn{2}{c}{Ridge SVD} & \multicolumn{2}{c}{Ridge GD} & \multicolumn{2}{c}{Paired $\Delta$} \\
		\cmidrule(lr){4-5} \cmidrule(lr){6-7} \cmidrule(lr){8-9}
		Signal & Spectrum & $p$ & $\lambda\!\geq\!0$ & $\lambda\!\leq\!0$ & $\lambda\!\geq\!0$ & $\lambda\!\leq\!0$ & SVD$^{-}\!-\!$GD$^{-}$ & GD$^{+}\!-\!$GD$^{-}$ \\
		\midrule
		\multirow{8}{*}{\rotatebox[origin=c]{90}{Spectral Sparse}} & \multirow{2}{*}{Spiked+Flat$^\dagger$}
		& 500 & .611\,(.008) & .588\,(.008) & .611\,(.008) & \textbf{.512\,(.009)} & \textbf{+.077\,(.004)} & \textbf{+.100\,(.005)} \\
		&
		& 1000 & .808\,(.008) & .701\,(.008) & .796\,(.011) & \textbf{.608\,(.011)} & \textbf{+.093\,(.005)} & \textbf{+.188\,(.010)} \\[2pt]
		& \multirow{2}{*}{Power-law ($\alpha_{\rm PL}\!=\!0.5$)}
		& 500 & .740\,(.010) & .683\,(.009) & .739\,(.010) & \textbf{.596\,(.008)} & \textbf{+.088\,(.005)} & \textbf{+.143\,(.006)} \\
		&
		& 1000 & .916\,(.009) & .759\,(.008) & .908\,(.010) & \textbf{.657\,(.007)} & \textbf{+.102\,(.004)} & \textbf{+.251\,(.008)} \\[2pt]
		& \multirow{2}{*}{Power-law ($\alpha_{\rm PL}\!=\!0.75$)}
		& 500 & .460\,(.005) & .456\,(.005) & .455\,(.006) & \textbf{.405\,(.004)} & \textbf{+.051\,(.002)} & \textbf{+.050\,(.003)} \\
		&
		& 1000 & .499\,(.006) & .469\,(.005) & .499\,(.006) & \textbf{.405\,(.004)} & \textbf{+.064\,(.003)} & \textbf{+.094\,(.004)} \\[2pt]
		& \multirow{2}{*}{Power-law ($\alpha_{\rm PL}\!=\!1$)}
		& 500 & .378\,(.003) & .391\,(.002) & .361\,(.002) & \textbf{.343\,(.002)} & \textbf{+.048\,(.002)} & \textbf{+.018\,(.002)} \\
		&
		& 1000 & .379\,(.003) & .381\,(.003) & .367\,(.003) & \textbf{.341\,(.002)} & \textbf{+.040\,(.002)} & \textbf{+.025\,(.002)} \\
		\midrule
		\multirow{8}{*}{\rotatebox[origin=c]{90}{Source Cond.}} & \multirow{2}{*}{Spiked+Flat}
		& 500 & .584\,(.006) & .571\,(.005) & .583\,(.006) & \textbf{.499\,(.005)} & \textbf{+.072\,(.005)} & \textbf{+.084\,(.006)} \\
		&
		& 1000 & .695\,(.009) & .651\,(.006) & .685\,(.009) & \textbf{.568\,(.005)} & \textbf{+.084\,(.005)} & \textbf{+.117\,(.009)} \\[2pt]
		& \multirow{2}{*}{Power-law ($\alpha_{\rm PL}\!=\!0.5$)}
		& 500 & .732\,(.005) & .719\,(.005) & .732\,(.005) & \textbf{.682\,(.006)} & \textbf{+.037\,(.003)} & \textbf{+.050\,(.004)} \\
		&
		& 1000 & .796\,(.007) & .754\,(.006) & .794\,(.007) & \textbf{.708\,(.006)} & \textbf{+.046\,(.004)} & \textbf{+.086\,(.005)} \\[2pt]
		& \multirow{2}{*}{Power-law ($\alpha_{\rm PL}\!=\!0.75$)}
		& 500 & .473\,(.006) & .472\,(.005) & .469\,(.006) & \textbf{.436\,(.005)} & \textbf{+.036\,(.002)} & \textbf{+.033\,(.002)} \\
		&
		& 1000 & .497\,(.005) & .477\,(.005) & .497\,(.005) & \textbf{.444\,(.005)} & \textbf{+.033\,(.003)} & \textbf{+.053\,(.004)} \\[2pt]
		& \multirow{2}{*}{Power-law ($\alpha_{\rm PL}\!=\!1$)}
		& 500 & .375\,(.003) & .391\,(.002) & .361\,(.003) & \textbf{.348\,(.002)} & \textbf{+.043\,(.002)} & \textbf{+.013\,(.001)} \\
		&
		& 1000 & .374\,(.003) & .377\,(.003) & .364\,(.003) & \textbf{.350\,(.002)} & \textbf{+.028\,(.002)} & \textbf{+.014\,(.002)} \\
		\bottomrule
	\end{tabular}
	}
\end{table}

\begin{figure}[t]
	\centering
	\includegraphics[width=0.8\linewidth]{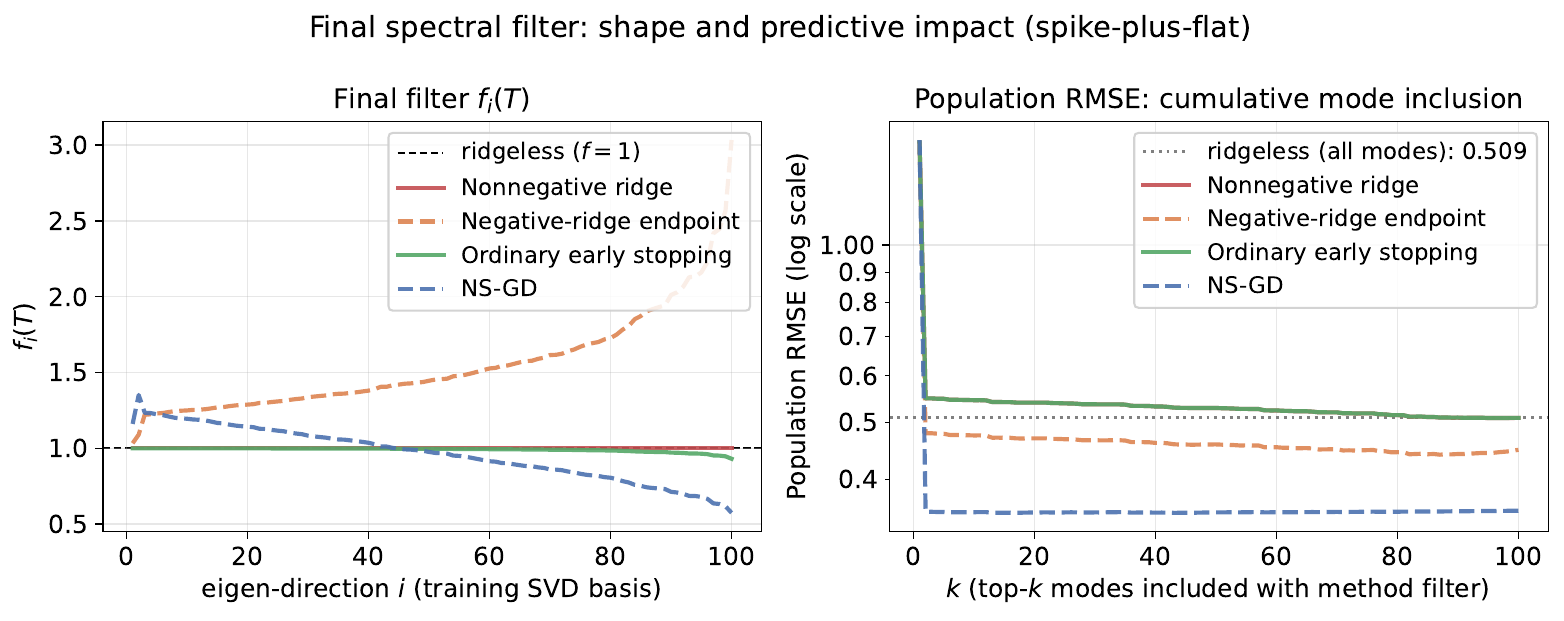}
	\caption{
		Spiked+Flat representative run: final spectral filters and cumulative
		population RMSE in the training SVD basis, all four methods
		validation-selected.
		Left: learned filters \(f_i(T)\).
		Right: cumulative top-\(k\) mode inclusion, reporting exact conditional
		population RMSE.  NS-GD attains the lowest risk; the negative-ridge
		endpoint improves over the nonnegative baselines but remains higher.
	}
	\label{fig:finalfilter}
\end{figure}

\noindent\textbf{Filter interpretation.}
Figure~\ref{fig:finalfilter} shows the final spectral filters and cumulative
population RMSE for a representative Spiked+Flat run with all four methods
validation-selected.
The nonnegative methods remain essentially ridgeless or shrinkage rules, with
filters near or below \(1\), whereas the two negative methods enter the
anti-shrinkage regime; however, their shapes differ markedly.
The static negative-ridge endpoint is tail-heavy: its anti-shrinkage grows on
the lower empirical directions, toward its pole below \(\widehat\mu_{\min}^+\).
The early-stopped NS-GD path instead concentrates anti-shrinkage on the leading
head directions while remaining controlled on the lower spectrum, using a
supercritical shift that a stable endpoint cannot reach without meeting the
unstable pole.
The right panel shows that including the top \(k\) modes cumulatively with the
NS-GD filter yields the largest population-RMSE reduction among the tuned
methods, well below the negative-ridge endpoint and the nonnegative baselines.
Figure~\ref{fig:filtercomparisonspikedflat} provides the corresponding
filter-trajectory diagnostics.

\begin{figure}[t]
	\centering
	\includegraphics[width=0.8\linewidth]{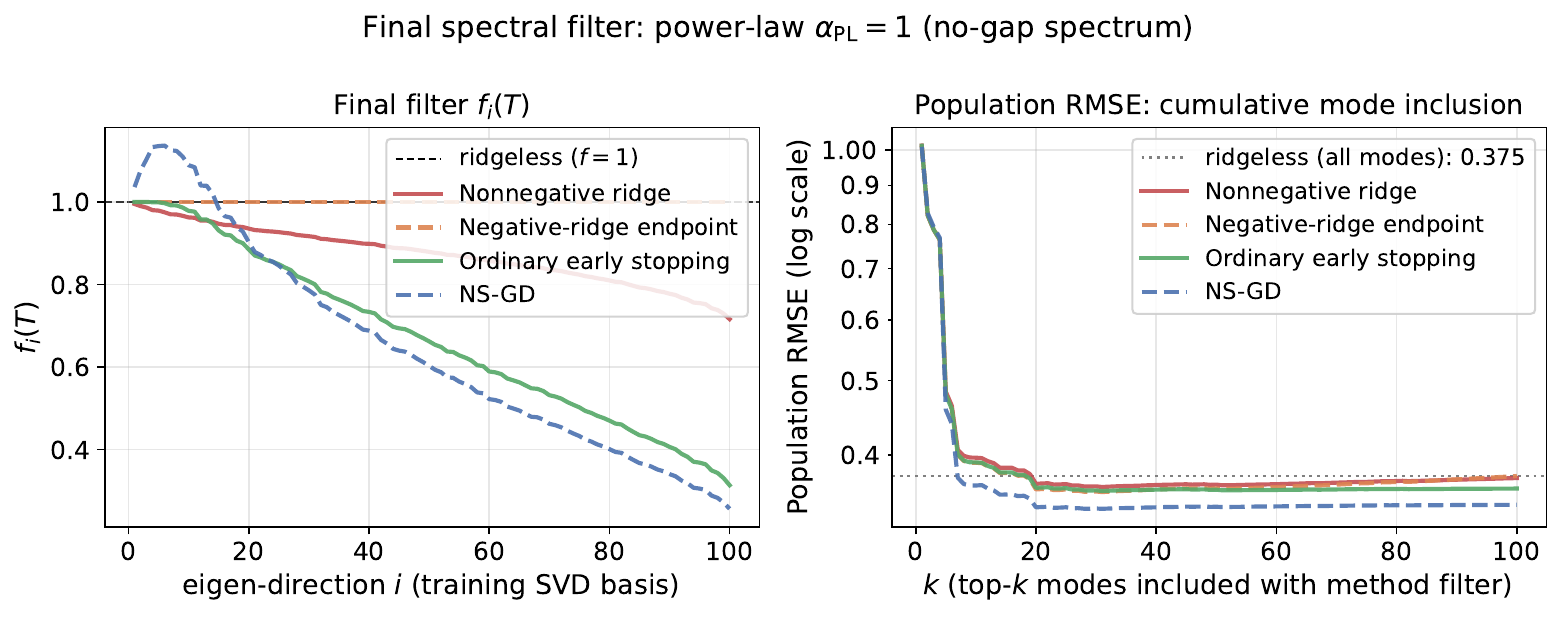}
	\caption{
		Power-law \(\alpha_{\rm PL}=1\) diagnostic (no-gap spectrum), all four
		methods validation-selected.
		Left: learned filters \(f_i(T)\).  Right: cumulative top-\(k\) mode
		inclusion, reporting exact conditional population RMSE.
		The negative-ridge endpoint collapses to ridgeless (its sign gives no
		gain), while NS-GD adds only modest head anti-shrinkage and shrinks the
		tail; all methods cluster, so the advantage here is a shape effect rather
		than a sign effect.
	}
	\label{fig:finalfilter-powerlaw-one}
\end{figure}

\noindent\textbf{Power-law \(\alpha_{\rm PL}=1\) as an endpoint failure mode.}
Figure~\ref{fig:finalfilter-powerlaw-one} highlights a scale-diagnostic regime
where the sign of the ridge parameter plays a secondary role.  For a power-law spectrum
\(\lambda_i\asymp i^{-1}\), the Tsigler--Bartlett tail floor satisfies
\(a_k=T_k/n\approx \log(p/k)/n\) and
\(S_k=\sum_{i>k}\lambda_i^2\approx 1/k\).  The endpoint proxy has residual
scale \(b=a_k-\alpha\), risk
\(\Theta_Hb^2+\sigma_\varepsilon^2S_k/(nb^2)\), and optimizer
\(b_\star=(\sigma_\varepsilon^2S_k/(n\Theta_H))^{1/4}\).
In this proxy, negative ridge is useful when \(b_\star<a_k\).  In the
power-law \(\alpha_{\rm PL}=1\) rows this condition can fail: the implicit tail
floor can be too small to warrant cancellation, and a stable endpoint may pay more
variance than it saves in head bias.  The finite-time signed path can still put
modest anti-shrinkage on leading directions while keeping lower-spectrum
directions below one, so its gain is a shape effect beyond sign.

\noindent\textbf{Trajectory diagnostics.}
Figure~\ref{fig:filtercomparisonspikedflat} shows the finite-time mechanism
directly: the NS-GD head-mode filters cross into anti-shrinkage and the head
error reaches its minimum at the validation-selected stop, before lower-spectrum
modes accumulate endpoint-style amplification.  Faded horizontal lines are the
static SVD baselines.

\begin{figure}[t]
	\centering
	\includegraphics[width=\linewidth]{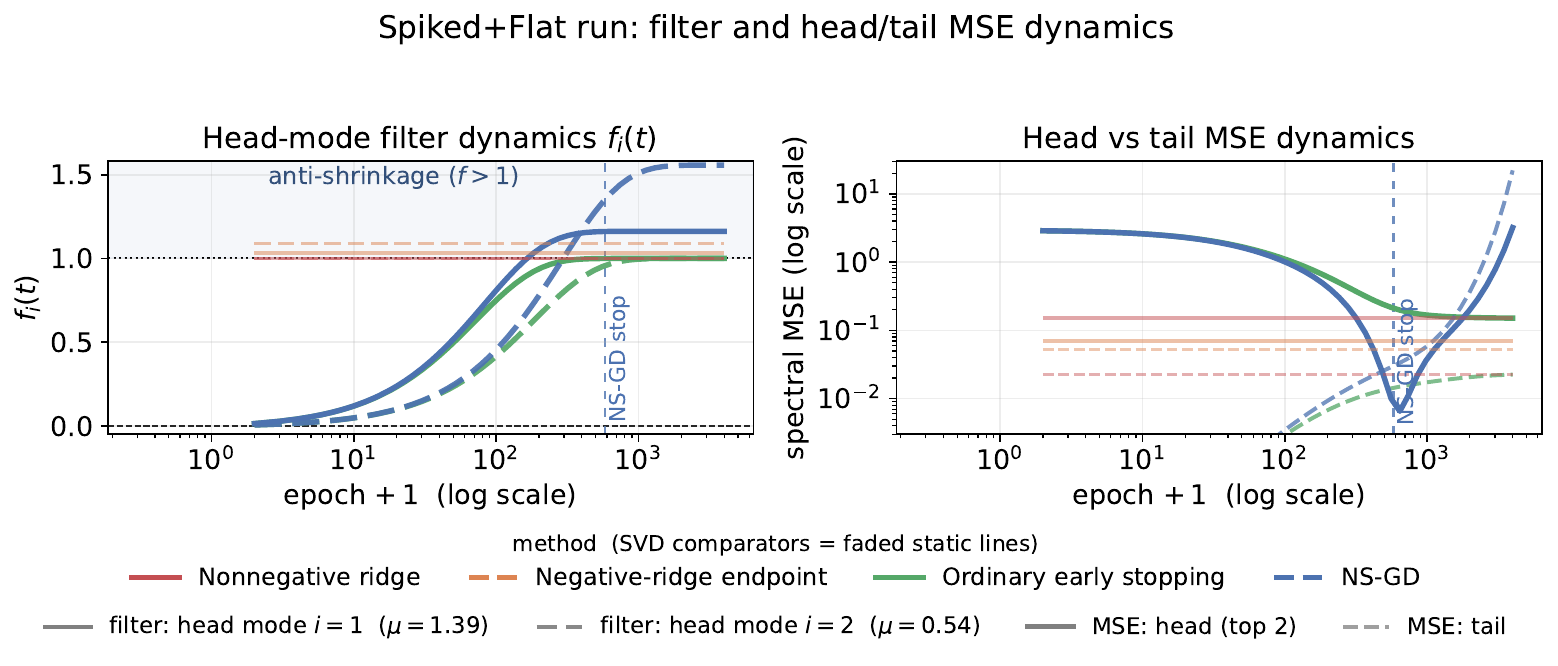}
	\caption{
		Spiked+Flat representative run: NS-GD versus baselines over training.
		Left: head-mode filter dynamics \(f_i(t)\); NS-GD crosses into the
		shaded anti-shrinkage region (\(f>1\)) on the leading directions, while
		the positive methods plateau at or below one.  Right: head (solid) and
		tail (dashed) spectral MSE; the NS-GD head error reaches its minimum at
		the validation-selected stop (dashed vertical line), before lower-spectrum
		exposure grows.  Faded horizontal lines are the static SVD baselines.
	}
	\label{fig:filtercomparisonspikedflat}
\end{figure}

\section{Discussion}
\label{sec:discussion}

We proved that finite-time negative-shifted GD realizes spectral shapes beyond
stable scalar negative-ridge endpoints, and quantified the resulting
prediction-risk separation on gapped Gaussian head--tail models.  We close by
discussing the scope of these guarantees and the directions they open.

\noindent\textbf{No-gap spectra: the primary future extension.}
Our theorems assume a \emph{gapped} spectrum: a resolved head of dimension
\(k=o(n)\) separated from a high-effective-rank tail, either flat (Case~I) or
heterogeneous (Case~II).  The
central open direction is the \emph{no-gap} regime, in which the eigenvalues
form a single continuous ladder with no head/tail separation, as in the
power-law robustness rows of Table~\ref{tab:rmse-deterministic} and the
diagnostic of Figure~\ref{fig:finalfilter-powerlaw-one}.  The truncated
power-law Corollary~\ref{cor:journal-multispike-power-law} does not cover this
regime: it permits power-law decay within the tail but retains
\(\lambda_{k+1}/\lambda_-\to0\), hence a genuine head/tail gap.  A full
no-gap theory would replace the discrete head/tail split by a matched
signal-bearing shell read at the floor-shifted empirical location, and would
require a gap-free continuum-edge alignment in place of eigengap perturbation;
the experiments already indicate that
finite-time signed paths remain useful there, where the sign of the optimal
ridge plays a secondary role and the gain is a shape effect
(Figure~\ref{fig:finalfilter-powerlaw-one}).

\noindent\textbf{Kernel ridge extension.}
The filter calculus itself is dimension-free: by the representer theorem,
NS-GD on kernel ridge regression acts through the same removable-pole filter
on the empirical kernel spectrum, and the exact design-conditional oracle of
Section~\ref{sec:exact-conditional-target} applies verbatim with the kernel
Gram matrix in place of \(X^\top X/n\).  What changes is the population
geometry.  Fixed-domain, trace-class Mercer spectra keep the tail trace
bounded, so the implicit floor \(a=T_k/n\) vanishes at rate \(1/n\) and the
correction it invites is asymptotically negligible---a boundary case rather
than a counterexample.  High-dimensional kernel regimes whose effective
spectra carry a divergent tail trace admit genuine floors at signal-bearing
scales, and there the endpoint--path asymmetry should persist.  Because
Mercer spectra are typically gap-free, this extension shares the no-gap
prerequisite above; moreover, kernel-side model-transfer errors enter
prediction risk additively, so a proved separation transfers only when that
error is a small fraction of the separation gap.  We leave kernel ridge
NS-GD as a companion direction.

\noindent\textbf{Discrete overshoot comparator.}
The discrete common-spike identity holds for every iteration count \(m\), but
at \(m=1\) the filter is independent of the signed level and reduces to a single
large-step ordinary gradient update.  More generally, large-step positive GD can
overshoot when parity and step size align with the target scale.  The present
discrete upper bounds, including
Corollary~\ref{cor:journal-multispike-discrete}, are therefore not a
lower-envelope separation from all over-relaxed positive-GD polynomials.

\noindent\textbf{Code and data availability.}
All experiments are synthetic and reproducible from fixed random seeds; no
external data are used.  The supplementary material provides self-contained
NumPy/Matplotlib code that regenerates every figure from scratch
(Figures~\ref{fig:motivation-supercritical}--\ref{fig:filtercomparisonspikedflat}),
together with a reference implementation of the validation-selected comparison
protocol used in Table~\ref{tab:rmse-deterministic} and a \texttt{README}
mapping each script to the item it produces.  Running the four figure scripts
reproduces all plotted curves and their emitted per-replication CSV summaries
in a few minutes on a standard laptop CPU.

\appendix

\section{Filter recursions and risk decomposition}
\label{app:filters-risk}

\begin{proof}[Proof of \eqref{eq:risk-decomposition}]
Since \(v_i^\top b=\mu_i\theta_i^\star+v_i^\top X^\top\varepsilon/n\),
\[
v_i^\top\widehat\beta_f
=
f_i\theta_i^\star
+
\frac{f_i}{\mu_i}v_i^\top X^\top\varepsilon/n .
\]
The prediction error in direction \(v_i\) is
\[
v_i^\top(\widehat\beta_f-\beta^\star)
=
-(1-f_i)\theta_i^\star
+
\frac{f_i}{\mu_i}v_i^\top X^\top\varepsilon/n .
\]
The noise term has conditional variance
\[
\frac{f_i^2}{\mu_i^2}
\operatorname{Var}\!\left(v_i^\top X^\top\varepsilon/n\mid X\right)
=
\frac{f_i^2}{\mu_i^2}\frac{\sigma_\varepsilon^2}{n}\mu_i
=
\frac{\sigma_\varepsilon^2}{n}\frac{f_i^2}{\mu_i}.
\]
Multiplying by the prediction weight \(\mu_i\) and summing over nonzero
empirical directions gives \eqref{eq:risk-decomposition}.
\end{proof}

\noindent\textbf{Discrete NS-GD filter.}
The experiments use full-batch discrete GD.  The exact conditional
population-risk formulas apply to the resulting discrete filter vector; the
gradient-flow notation is used only for analytic envelopes. With step size
\(\eta>0\) and negative shift \(\nu\), the row-space coordinate recursion gives the finite
\(q\)-step filter
\[
f_{\nu,\eta,q}^{\rm disc}(\mu)
=
\frac{\mu}{\mu-\nu}
\left\{1-\left(1-\eta(\mu-\nu)\right)^q\right\},
\]
with the continuous extension at \(\mu=\nu\). For \(0<\eta(\mu-\nu)<1\) it gives
a monotone discretization of the NS-GF filter. For larger nondivergent
steps or for \(\mu<\nu\), the finite-time filter can be oscillatory or
transient; the endpoint interpretation applies to stable directions
\(\mu>\nu\), while finite-time horizon budgets are the relevant objects for
unstable directions.  Corollary~\ref{cor:journal-multispike-discrete} gives
the regime-specific discrete transfer guarantee used in the paper.

\section{Proof of conditional population-risk identity}
\label{app:conditional-pop-risk}

\begin{proof}[Proof of Proposition~\ref{prop:conditional-pop-risk}]
Write \(b=\widehat\Sigma\beta^\star+X^\top\varepsilon/n\).
Since \(\widehat\Sigma\widehat v_i=\widehat\mu_i\widehat v_i\),
\[
\frac{\widehat v_i^\top b}{\widehat\mu_i}
=
\widehat v_i^\top\beta^\star
+
\frac{\widehat v_i^\top X^\top\varepsilon}{n\widehat\mu_i}.
\]
Therefore
\[
\widehat\beta_f
=
\widehat V\operatorname{diag}(f)\widehat V^\top\beta^\star
+
\sum_{i=1}^r
f_i
\frac{\widehat v_i^\top X^\top\varepsilon}{n\widehat\mu_i}
\widehat v_i .
\]
The conditional bias is
\((\widehat V\operatorname{diag}(f)\widehat V^\top-I)\beta^\star\).
For the noise coefficients,
\[
\operatorname{Cov}\!\left(
\frac{\widehat v_i^\top X^\top\varepsilon}{n\widehat\mu_i},
\frac{\widehat v_j^\top X^\top\varepsilon}{n\widehat\mu_j}
\middle|X
\right)
=
\frac{\sigma_\varepsilon^2}{n}
\frac{\mathbf 1\{i=j\}}{\widehat\mu_i},
\]
because the empirical singular directions diagonalize \(X^\top X/n\). Taking
the population prediction norm of the noise part gives
\[
\frac{\sigma_\varepsilon^2}{n}
\sum_{i=1}^r
f_i^2
\frac{\widehat v_i^\top\Sigma\widehat v_i}{\widehat\mu_i}.
\]
The cross term vanishes after conditioning on \(X\), proving the identity.
\end{proof}

\section{Proofs of fixed-design signed-filter facts}
\label{app:fixed-design-proofs}

\subsection{Diagnostic filter statements}

\begin{lemma}[Anti-shrinkage requires a negative shift]
\label{lem:antishrinkage-requires-negative-shift}
For every nonnegative ridge gradient-flow path \(f^+_{\lambda,t}\) from
Proposition~\ref{prop:filter-geometry-signed-finite-time}, with
\(\lambda\ge0\), \(0\le f^+_{\lambda,t}(\mu)\le1\) for all \(\mu\ge0\) and
\(t\ge0\).  For ridgeless GD with monotone stable step size
\(0\le\alpha\mu\le1\), the discrete filter
\(f_{\mu,q}=1-(1-\alpha\mu)^q\) also satisfies \(0\le f_{\mu,q}\le1\).
\end{lemma}

\begin{lemma}[Finite-time paths as effective spectral ridge]
\label{lem:finite-time-effective-spectral-ridge}
For \(t>0\), \(\lambda\ge0\), and \(\mu+\lambda>0\), positive ridge gradient
flow has a positive eigenvalue-dependent effective scale:
\[
f^+_{\lambda,t}(\mu)
=
\frac{\mu}{\mu+x^+_{\rm eff}(\mu,t)},
\qquad
x^+_{\rm eff}(\mu,t)
=
\lambda+
\frac{\mu+\lambda}{e^{(\mu+\lambda)t}-1}.
\]
A scalar signed-ridge endpoint has constant scale \(x_{\rm end}=-\nu\).  For
signed ridge flow with shift \(\nu>0\) and \(\mu\ne\nu\),
\[
f^-_{\nu,t}(\mu)
=
\frac{\mu}{\mu+x^-_{\rm eff}(\mu,t)},
\qquad
x^-_{\rm eff}(\mu,t)
=
-\nu+
\frac{\mu-\nu}{e^{(\mu-\nu)t}-1},
\]
with continuous extension \(x^-_{\rm eff}(\nu,t)=1/t-\nu\).
\end{lemma}

\begin{lemma}[Same-shift stable endpoint has the larger scalar filter]
\label{lem:signed-endpoint-dominates-path}
Fix \(\nu>0\). For every stable empirical direction \(\mu>\nu\) and every
\(t>0\), the filters in
\eqref{eq:signed-endpoint-filter}--\eqref{eq:signed-gf-filter} satisfy
\(0<f_{\nu,t}(\mu)<A_\nu(\mu)\).
\end{lemma}

\begin{lemma}[Scalar risk gap to the same-shift stable endpoint]
\label{lem:finite-signed-time-improves-own-endpoint}
Fix a stable direction \(\mu>\nu\). Let
\(A=\mu/(\mu-\nu)\), \(\eta(t)=e^{-(\mu-\nu)t}\), \(f_t=A(1-\eta)\), and
\(f_\infty=A\).
For the scalar coordinate risk
\(R_i(f)=s_i(1-f)^2+v_i f^2\) with \(s_i,v_i\ge0\), we have
\[
R_i(f_t)-R_i(f_\infty)
=
A^2\eta
\left[
(s_i+v_i)\eta
-
2\left(v_i+s_i\frac{\nu}{\mu}\right)
\right].
\]
Consequently, if \(s_i+v_i>0\) and
\(0<\eta(t)<2\{v_i+s_i\nu/\mu\}/(s_i+v_i)\),
then \(R_i(f_t)<R_i(f_\infty)\).
\end{lemma}

\begin{remark}[Positive contrast: why the positive case is different]
For positive ridge flow, finite time trades variance reduction against extra
under-shrinkage bias, while the signed case backs away from terminal
signed-endpoint overshoot.  The scalar identity is derived below and is
diagnostic.
\end{remark}

\begin{proof}[Proof of Lemma~\ref{lem:antishrinkage-requires-negative-shift}]
For \(\lambda\ge0\),
\(0\le 1-e^{-(\mu+\lambda)t}\le1\) and \(0\le\mu/(\mu+\lambda)\le1\),
with the convention that the filter is zero when \(\mu=\lambda=0\). Hence
\(0\le f^+_{\lambda,t}(\mu)\le1\). For ridgeless GD, stability gives
\(0\le1-\alpha\mu\le1\), so \(0\le(1-\alpha\mu)^q\le1\) and therefore
\(0\le f_{\mu,q}=1-(1-\alpha\mu)^q\le1\).
\end{proof}

\begin{proof}[Proof of Lemma~\ref{lem:finite-time-effective-spectral-ridge}]
For the positive path, solve, for \(\mu>0\),
\[
\frac{\mu}{\mu+x^+_{\rm eff}(\mu,t)}
=
\frac{\mu}{\mu+\lambda}
\left(1-e^{-(\mu+\lambda)t}\right)
\]
for \(x^+_{\rm eff}\). This gives
\[
x^+_{\rm eff}(\mu,t)
=
\frac{\mu+\lambda}{1-e^{-(\mu+\lambda)t}}-\mu
=
\lambda+
\frac{\mu+\lambda}{e^{(\mu+\lambda)t}-1},
\]
which is positive when \(\lambda\ge0\) and \(t>0\). The formula extends to
\(\mu=0\) by direct substitution.

The scalar signed endpoint identity is immediate:
\(\mu/(\mu-\nu)=\mu/(\mu+x_{\rm end})\) with \(x_{\rm end}=-\nu\).
Since \(x_{\rm end}\) is constant in \(\mu\), a scalar endpoint assigns one
effective sign across spectral regions.

The signed calculation is the same with \(\lambda=-\nu\):
\[
x^-_{\rm eff}(\mu,t)
=
\frac{\mu-\nu}{1-e^{-(\mu-\nu)t}}-\mu
=
-\nu+
\frac{\mu-\nu}{e^{(\mu-\nu)t}-1}.
\]
Taking \(\mu\to\nu\) gives the continuous extension
\(x^-_{\rm eff}(\nu,t)=1/t-\nu\). The displayed formula shows that the signed
effective scale may be negative; for example, on large enough \(\mu\) after
the time when \(f_{\nu,t}(\mu)\) crosses the ridgeless
level, it is negative, while the finite-time denominator can remain positive
on lower eigenvalues.
\end{proof}

\begin{proof}[Proof of Lemma~\ref{lem:signed-endpoint-dominates-path}]
If \(\mu>\nu\), then \(\mu/(\mu-\nu)>0\) and
\(0<1-e^{-(\mu-\nu)t}<1\) for every \(t>0\); multiplying the two gives the
claim.
\end{proof}

\begin{proof}[Proof of Lemma~\ref{lem:finite-signed-time-improves-own-endpoint}]
Since \(A=\mu/(\mu-\nu)\),
\[
1-f_\infty
=
1-A
=
-\frac{\nu}{\mu-\nu}
=
-A\frac{\nu}{\mu},
\qquad
1-f_t
=
1-A+A\eta
=
A\left(\eta-\frac{\nu}{\mu}\right).
\]
Therefore
\(s_i\{(1-f_t)^2-(1-f_\infty)^2\}
=A^2s_i(\eta^2-2\eta\nu/\mu)\)
and
\(v_i(f_t^2-f_\infty^2)
=A^2v_i\{(1-\eta)^2-1\}
=A^2v_i(\eta^2-2\eta)\).
Adding the two identities gives the risk-gap identity stated in
Lemma~\ref{lem:finite-signed-time-improves-own-endpoint}.
The displayed sufficient condition is exactly the condition that the bracket
be negative. If \(s_i=0\) and \(v_i>0\), the condition becomes
\(0<\eta<2\), which holds for every finite \(t>0\). At the time when the
finite-time filter crosses the ridgeless level, equivalently
\(\eta=\nu/\mu\), and since \(\mu>\nu\),
\((s_i+v_i)\nu/\mu<2(v_i+s_i\nu/\mu)\)
whenever \(s_i+v_i>0\), so the finite-time path again improves over the same
endpoint. Finally,
\[
\left.
\frac{\partial}{\partial\eta}
\{R_i(f_t)-R_i(f_\infty)\}
\right|_{\eta=0}
=
-2A^2\left(v_i+s_i\frac{\nu}{\mu}\right),
\]
which is strictly negative unless \(s_i=v_i=0\).
\end{proof}

\begin{proof}[Proof of Corollary~\ref{cor:signed-gf-superlevel-prefix}]
The residual of the signed-flow filter is
\(1-f_{\nu,t}(\mu)
=\{\mu e^{-(\mu-\nu)t}-\nu\}/(\mu-\nu)\),
with the continuous interpretation at \(\mu=\nu\). For \(\mu\ne\nu\), the
condition \(f_{\nu,t}(\mu)>1\) is equivalent to
\[
\mu e^{-(\mu-\nu)t}<\nu
\quad\text{if }\mu>\nu,
\qquad
\mu e^{-(\mu-\nu)t}>\nu
\quad\text{if }\mu<\nu,
\]
and both inequalities reduce to \(t>g_\nu(\mu)\). At \(\mu=\nu\),
\(f_{\nu,t}(\nu)=\nu t\), so the same statement holds with
\(g_\nu(\nu)=1/\nu\).

For \(\mu\ne\nu\),
\(g_\nu'(\mu)
=\{1-\nu/\mu-\log(\mu/\nu)\}/(\mu-\nu)^2\).
The numerator is negative for \(\mu\ne\nu\), because
\(\log x>1-1/x\) for \(x\ne1\). Hence \(g_\nu\) is strictly decreasing on
\((0,\infty)\) after filling in the continuous value at \(\mu=\nu\). Sorting
the empirical eigenvalues in decreasing order therefore makes the above-one
set a prefix.
\end{proof}

\begin{proof}[Proof of Corollary~\ref{cor:signed-gd-superlevel-prefix}]
For \(z=\eta(\mu-\nu)<1\), define
\[
t_{\rm eff}(\mu)
:=
m\eta\,\frac{-\log(1-z)}{z},
\]
with value \(m\eta\) at \(z=0\).  The geometric-sum identity gives
\(f_{\nu,\eta,m}^{\rm disc}(\mu)
=f_{\nu,t_{\rm eff}(\mu)}(\mu)\).
Moreover,
\[
\frac{-\log(1-z)}{z}
=
\int_0^1\frac{ds}{1-sz}
\]
is strictly increasing for \(z<1\), so \(t_{\rm eff}(\mu)\) increases in
\(\mu\), whereas \(g_\nu(\mu)\) in
Corollary~\ref{cor:signed-gf-superlevel-prefix} strictly decreases.
Thus \(f_{\nu,\eta,m}^{\rm disc}(\mu)>1\), equivalently
\(t_{\rm eff}(\mu)>g_\nu(\mu)\), can cross only once as \(\mu\) increases.
\end{proof}

\section{Full-risk oracle proofs and filter diagnostics}
\label{app:full-risk-oracle-proofs}

This appendix collects deterministic conditional tools that support the
full-risk story and the filter diagnostics used in the main text.

\begin{proof}[Proof of Theorem~\ref{thm:finite-grid-validation-oracle}]
Condition on \(\mathcal D_{\rm tr}\).  For each \(g\in\mathcal G\), define
\[
L_g
:=
\sigma_\varepsilon^2+
\|\widehat\beta_g-\beta^\star\|_\Sigma^2,
\qquad
\widehat L_g
:=
\frac1{n_{\rm val}}
\sum_{i=1}^{n_{\rm val}}
(y_i^{\rm val}-x_i^{{\rm val}\top}\widehat\beta_g)^2 .
\]
The validation residuals for a fixed \(g\) are independent centered Gaussian
variables with variance \(L_g\).  Hence
\(n_{\rm val}\widehat L_g/L_g\sim\chi^2_{n_{\rm val}}\).
The Laurent--Massart bounds \citep[Lemma~1]{LaurentMassart2000} imply, for
every \(z>0\),
\[
\Pr\!\left(
\frac{\widehat L_g}{L_g}
>1+2\sqrt{\frac z{n_{\rm val}}}+\frac{2z}{n_{\rm val}}
\ \middle|\ \mathcal D_{\rm tr}
\right)\le e^{-z}
\]
and
\[
\Pr\!\left(
\frac{\widehat L_g}{L_g}
<1-2\sqrt{\frac z{n_{\rm val}}}
\ \middle|\ \mathcal D_{\rm tr}
\right)\le e^{-z}.
\]
Taking \(z=z_{\rm val}\) and applying a union bound to both tails for all
\(g\in\mathcal G\) gives, with conditional probability at least \(1-\delta\),
\[
(1-\epsilon_{\rm val})L_g
\le\widehat L_g
\le(1+\epsilon_{\rm val})L_g
\qquad\text{for every }g\in\mathcal G.
\tag{Uniform validation event}
\label{eq:uniform-validation-event}
\]
Dependence among the losses for different \(g\) is immaterial to this union
bound.  Let \(g^\star\) minimize \(L_g\) over \(\mathcal G\).  On the uniform
event in \eqref{eq:uniform-validation-event}, empirical optimality gives
\[
(1-\epsilon_{\rm val})L_{\widehat g}
\le\widehat L_{\widehat g}
\le\widehat L_{g^\star}
\le(1+\epsilon_{\rm val})L_{g^\star}.
\]
This proves the first line of
\eqref{eq:finite-grid-validation-oracle}; subtracting
\(\sigma_\varepsilon^2\) proves the second.

For the final assertion, take the fixed theoretical candidate \(g_0\) in the
first inequality.  Its realized training-sample excess risk is
\(o_{\mathbb P}(B_n)\) whenever its nonnegative design-conditional mean in the
theorem statement is \(o_{\mathbb P}(B_n)\), by conditional Markov's
inequality.  For the spectral candidates studied in this paper, that mean is
exactly \(\Risk_{\rm pop}(f_{g_0}\mid X_{\rm tr})\).  Moreover,
\[
\Pr\!\left(
\|\widehat\beta_{g_0}-\beta^\star\|_\Sigma^2>\eta B_n
\right)
\le
\mathbb E\!\left[
\min\!\left\{
1,\,
\frac{
\mathbb E_{\varepsilon_{\rm tr}}
[\|\widehat\beta_{g_0}-\beta^\star\|_\Sigma^2\mid X_{\rm tr}]
}{\eta B_n}
\right\}
\right]
\longrightarrow0
\]
for every fixed \(\eta>0\).  Also,
\[
\frac{1+\epsilon_{\rm val}}{1-\epsilon_{\rm val}}-1
=
\frac{2\epsilon_{\rm val}}{1-\epsilon_{\rm val}}.
\]
The validation event has probability tending to one when \(\delta=o(1)\);
the assumed validation remainder therefore gives the stated
\(o_{\mathbb P}(B_n)\) transfer.
\end{proof}

\begin{proposition}[Auxiliary endpoint and finite-time envelopes]
\label{prop:auxiliary-filter-envelopes}
The endpoint \(A_\alpha\) has a pole at \(\alpha=\widehat\mu_i\).  In
particular, if
\[
w_i:=\frac{\sigma_\varepsilon^2}{n}
\frac{\widehat v_i^\top\Sigma\widehat v_i}{\widehat\mu_i},
\qquad
V_X(A_\alpha):=\sum_{i=1}^r w_i A_\alpha(\widehat\mu_i)^2,
\]
then \(|\alpha-\widehat\mu_i|\le\varepsilon\) and \(w_i>0\) imply
\(V_X(A_\alpha)\ge w_i\widehat\mu_i^2/\varepsilon^2\).
By contrast, for every finite \((\nu,t)\),
\(|f_{\nu,t}(\mu)|\le \mu t e^{\nu t}\) and
\(|\partial_\mu f_{\nu,t}(\mu)|
\le te^{\nu t}+\widehat\mu_1(t^2/2)e^{\nu t}\)
on \(0\le\mu\le\widehat\mu_1\).
\end{proposition}

\begin{proof}[Proof of Proposition~\ref{prop:full-risk-oracle-geometry}]
The bias part of Proposition~\ref{prop:conditional-pop-risk} gives, with
\(\delta_i=f(\widehat\mu_i)-1\),
\((F-I)\beta^\star
=\sum_{i=1}^r\delta_iP_i\beta^\star-P_0\beta^\star\).
Subtracting the ridgeless-row-space bias, \(\|(P-I)\beta^\star\|_\Sigma^2\),
gives
\[
\|(F-I)\beta^\star\|_\Sigma^2-
\|(P-I)\beta^\star\|_\Sigma^2
=
\delta^\top H\delta-2m^\top\delta .
\]
The variance term changes from \(\sum_i v_i\) at ridgeless to
\(\sum_i v_i(1+\delta_i)^2\), hence its difference is
\(2v^\top\delta+\delta^\top D\delta\).  Combining the displays gives
\[
\Risk_{\rm pop}(f\mid X)
=
\Risk_{\rm pop}(\mathbf 1\mid X)+\delta^\top(H+D)\delta-2(m-v)^\top\delta .
\]

It remains to justify the completed square with the Moore--Penrose inverse.  Let
\(z\in\operatorname{null}(Q)\).  Since \(Q=H+D\succeq0\), we have
\(z^\top Hz=z^\top Dz=0\).  Put \(u_z=\sum_i z_iP_i\beta^\star\).  Then
\(z^\top Hz=\|u_z\|_\Sigma^2=0\), so Cauchy--Schwarz in the
\(\Sigma\)-seminorm gives
\(z^\top m=\langle u_z,P_0\beta^\star\rangle_\Sigma=0\).
Also, because \(D\) is nonnegative diagonal, \(z^\top Dz=0\) implies
\(Dz=0\), and therefore \(z^\top v=0\).  Hence \(z^\top g=0\).  Thus
\(g\perp\operatorname{null}(Q)\), equivalently
\(g\in\operatorname{range}(Q)\).  The minimizers of
\(\delta^\top Q\delta-2g^\top\delta\) are the solutions of \(Q\delta=g\), and
\(\delta^\star=Q^\dagger g\) is the minimum-norm solution.  Completing the
square yields
\(\delta^\top Q\delta-2g^\top\delta
=\|\delta-\delta^\star\|_Q^2-g^\top Q^\dagger g\),
which proves the oracle identity.  Finally, set \(\delta=\epsilon e_i\).
The risk difference from ridgeless is
\(\epsilon^2Q_{ii}-2\epsilon g_i
=(H_{ii}+v_i)\epsilon^2-2(m_i-v_i)\epsilon\).
Differentiating at zero gives \(-2(m_i-v_i)\), proving the local criterion.
\end{proof}

\begin{proof}[Proof of Proposition~\ref{prop:auxiliary-filter-envelopes}]
The endpoint variance statement is immediate from the conditional variance term
in Proposition~\ref{prop:conditional-pop-risk}.  If
\(|\alpha-\widehat\mu_i|\le\varepsilon\), then
\(A_\alpha(\widehat\mu_i)^2
=\widehat\mu_i^2/(\widehat\mu_i-\alpha)^2
\ge\widehat\mu_i^2/\varepsilon^2\).
For the finite-time path, the integral form in \eqref{eq:signed-gf-filter}
with \(e^{-(\mu-\nu)s}\le e^{\nu s}\le e^{\nu t}\) for \(0\le s\le t\) and
\(\mu\ge0\) gives \(|f_{\nu,t}(\mu)|\le\mu t e^{\nu t}\).  Differentiating
under the integral gives
\[
\partial_\mu f_{\nu,t}(\mu)
=
\int_0^t e^{-(\mu-\nu)s}\,ds
-
\mu\int_0^t s e^{-(\mu-\nu)s}\,ds,
\]
so on \([0,\widehat\mu_1]\),
\(|\partial_\mu f_{\nu,t}(\mu)|
\le te^{\nu t}+\widehat\mu_1(t^2/2)e^{\nu t}\).
\end{proof}

\noindent\textbf{Endpoint derivative at ridgeless.}
For the scalar endpoint, \(A_\alpha(\widehat\mu_i)-1=
\alpha/\widehat\mu_i+O(\alpha^2)\).  Substituting this displacement into the
exact full-risk expansion gives
\[
\left.
\frac{d}{d\alpha}\Risk_{\rm pop}(A_\alpha\mid X)
\right|_{\alpha=0}
=
-2\sum_{i=1}^r\frac{m_i-v_i}{\widehat\mu_i},
\]
because the quadratic \(\delta^\top Q\delta\) term is second order.

\section{Proofs for the Gaussian spike-model cases}
\label{app:gaussian-spike-case-proofs}

The proofs are organized by the two regimes in
Section~\ref{sec:gaussian-spike-models}.  Each begins from the full Gaussian
design and derives the concentration events it uses.  Case I combines the
Tsigler--Bartlett head--tail decomposition with direct finite-time Duhamel
control.  Case II retains the Tsigler--Bartlett trace and squared-spectrum
tail quantities, but replaces the non-closing global perturbation step by
localized invariant-graph and within-head Duhamel control.

\subsection{Case I: MP barrier and common-spike separation}
\label{app:journal-spike-flat-proofs}

Throughout this section, \(c,C\in(0,\infty)\) denote constants independent of
\(n\) and \(\lambda_T\).  Put
\[
r_T:=\frac{w_T}{\lambda_h}
\asymp\sqrt{\lambda_T/\lambda_h},
\qquad
K_{2,T}:=\frac1nX_T\Sigma_TX_T^\top=\lambda_TK_T .
\]
Since \(n\lambda_T/\lambda_h\to\infty\), we have
\(\sqrt n\,r_T\to\infty\).  Fix the deterministic tolerance
\[
\omega_{H,n}:=(\sqrt n\,r_T)^{-1/2}\longrightarrow0,
\qquad
n^{-1/2}=o(\omega_{H,n}r_T).
\]
Also define
\(\nu_\star=a+\lambda_h\), \(t_\star=\lambda_h^{-1}\), and
\(K_0:=K_H+aI_n\).
The proof combines the Gaussian head--tail concentration backbone used in
\citet{Tsigler2023Benign} with direct finite-time Duhamel control.  All events
needed below are derived from Assumption~\ref{ass:journal-spike-flat}; no
external event package is assumed.

\begin{lemma}[Gaussian common-spike event]
\label{lem:journal-spike-event}
Under Assumption~\ref{ass:journal-spike-flat}, there is an event
\(\mathcal E_{\rm cs}\) with probability tending to one on which
\[
\operatorname{spec}(K_T)
\subset
[a-Cw_T,a+Cw_T],
\qquad
\frac1n\operatorname{tr}(K_T)\asymp a,
\qquad
K_{2,T}=\lambda_TK_T,
\]
and
\[
\left\|\lambda_h^{-1}G_H-I_k\right\|_{\rm op}
\le \omega_{H,n}r_T,
\qquad
G_H:=X_H^\top X_H/n .
\]
In particular,
\[
\operatorname{spec}(K)
\subset
\left[
a-Cw_T,\,
a+\lambda_h+Cw_T
\right].
\]
For
\[
M_\star
:=
\frac1nX_H^\top h_{\nu_\star,t_\star}(K)X_H,
\qquad
M_{\star,m}^{\rm disc}
:=
\frac1nX_H^\top
h_{\nu_\star,\eta_\star,m}^{\rm disc}(K)X_H,
\]
the same event satisfies
\[
\|M_\star-I_k\|_{\rm op}
+
\sup_{m\ge1}\|M_{\star,m}^{\rm disc}-I_k\|_{\rm op}
\lesssim r_T.
\tag{Common-spike head map}
\label{eq:common-spike-head-map-event}
\]
Writing \(K=U\operatorname{diag}(\mu_i)U^\top\), it further satisfies
\[
\sum_{i:\,|\mu_i-a|\le Cw_T}u_i^\top K_{2,T}u_i
\gtrsim
na\lambda_T .
\tag{Projected tail square mass}
\label{eq:common-spike-projected-tail-mass}
\]
\end{lemma}

\begin{proof}
Gaussian singular-value concentration
\citep[Chapter~4]{Vershynin2018High-dimensional} gives
\[
\|K_T-aI_n\|_{\rm op}\lesssim w_T
\]
with probability tending to one, because
\(K_T=(\lambda_T/n)Z_TZ_T^\top\),
\(a=\gamma_T\lambda_T\), and \(\gamma_T\to\infty\).  It also gives
\(\operatorname{tr}(K_T)/n\asymp a\).  Since \(k=O(1)\), fixed-dimensional
Wishart concentration gives
\(\|\lambda_h^{-1}G_H-I_k\|_{\rm op}=O_{\mathbb P}(n^{-1/2})\).
Because \(n^{-1/2}=o(\omega_{H,n}r_T)\), the head bound in the
definition of \(\mathcal E_{\rm cs}\) holds with probability tending to one.
The full spectral band follows from Weyl's inequalities
\citep{HornJohnson2012MatrixAnalysis} and
\(\|K_H\|_{\rm op}=\|G_H\|_{\rm op}\).

For the head map, define
\(\psi_\star(x):=x h_{\nu_\star,t_\star}(a+x)\).
The sample-space/primal intertwining identity gives
\(n^{-1}X_H^\top h_{\nu_\star,t_\star}(K_0)X_H
=\psi_\star(G_H)\), and
\(a+\lambda_h-\nu_\star=0\) gives
\(\psi_\star(\lambda_h)=\lambda_h t_\star=1\).
On a fixed neighborhood of \(\lambda_h\),
\(\sup_x|\psi_\star'(x)|\lesssim\lambda_h^{-1}\); hence
\(\|\psi_\star(G_H)-I_k\|_{\rm op}
\lesssim\omega_{H,n}r_T\).

It remains to transfer from \(K_0\) to \(K\).  The standard Duhamel formula
for matrix exponentials \citep{Higham2008Functions} yields
\[
\|h_{\nu_\star,t_\star}(K)-h_{\nu_\star,t_\star}(K_0)\|_{\rm op}
\le
\|K_T-aI_n\|_{\rm op}
\int_0^{t_\star}s
e^{(\lambda_h+Cw_T)s}\,ds
\lesssim
\frac{w_T}{\lambda_h^2}.
\]
The squared operator norm of \(X_H/\sqrt n\) is \(O(\lambda_h)\).
Multiplication on both sides therefore gives an \(O(r_T)\) contribution.  This
proves the continuous part of
\eqref{eq:common-spike-head-map-event}.

For the discrete path, put
\(q_m(x):=h_{\nu_\star,\eta_\star,m}^{\rm disc}(a+x)\).
The exact floor-model head map is \(G_Hq_m(G_H)\).
At the common spike,
\(\lambda_hq_m(\lambda_h)=\lambda_hm\eta_\star=1\), and direct
differentiation of the finite geometric sum gives
\[
\sup_{m\ge1}\sup_{|x-\lambda_h|\le c\lambda_h}
\left|\frac{d}{dx}\{xq_m(x)\}\right|
\lesssim\lambda_h^{-1}.
\]
The matrix-polynomial telescoping identity, together with
\[
\max_{0\le\ell\le m}
\left(1+\frac{1+Cr_T}{m}\right)^\ell
\le C,
\]
gives
\[
\left\|
h_{\nu_\star,\eta_\star,m}^{\rm disc}(K)
-
h_{\nu_\star,\eta_\star,m}^{\rm disc}(K_0)
\right\|_{\rm op}
\lesssim
\frac{w_T}{\lambda_h^2}
\]
uniformly in \(m\ge1\).  The preceding head-map argument therefore proves the
discrete part of \eqref{eq:common-spike-head-map-event}.

Finally,
\(\operatorname{tr}(K_{2,T})\asymp na\lambda_T\) and
\(\|K_{2,T}\|_{\rm op}\lesssim a\lambda_T\).  Adding the rank-\(k\) positive
semidefinite head matrix can move at most \(k\) eigenvalues of \(K_T\) above
the displayed tail band.  The complementary spectral projector of \(K\)
therefore has rank \(O(k)\), and it can remove at most
\(O(k a\lambda_T)\) of \(K_{2,T}\)-mass.  Since \(k=O(1)\) and \(n\to\infty\),
\eqref{eq:common-spike-projected-tail-mass} follows.
\end{proof}

\begin{proof}[Proof of Proposition~\ref{prop:journal-mp-barrier}]
The tail Gram matrix is
\[
K_T=\frac{\lambda_T}{n}Z_TZ_T^\top,
\qquad
Z_T\in\R^{n\times d_T}.
\]
The assumptions imply \(\gamma_T=d_T/n\to\infty\).  Define the centered,
bulk-width normalized Wishart matrix
\(B_T:=(Z_TZ_T^\top-d_TI_n)/\sqrt{nd_T}\).
The semicircle law for centered sample covariance matrices in the regime
\(n/d_T\to0\) \citep{BaiYin1988Semicircle} implies that there are constants
\(c_0,q_0>0\) such that, with probability tending to one, at least \(q_0n\)
eigenvalues of \(B_T\) are at most \(-c_0\).  Since \(k=o(n)\), if the
eigenvalues are ordered decreasingly, this gives
\(\lambda_{n-k}(B_T)\le-c_0\).
The standard Gaussian singular-value bound
\citep[Theorem~4.6.1]{Vershynin2018High-dimensional} also gives
\(\|B_T\|_{\rm op}=O_{\mathbb P}(1)\).  Indeed, its two-sided bounds at
scale \(\sqrt{d_T}\pm C\sqrt n\), after squaring and centering, are
of order \(\sqrt{nd_T}\).

Because \(a=\lambda_Td_T/n\) and
\(w_T=\lambda_T\sqrt{d_T/n}\), the exact identity
\(K_T=aI_n+w_TB_T\)
therefore yields constants \(0<c<C<\infty\) such that, with probability
tending to one,
\(\lambda_n(K_T)\ge a-Cw_T\) and \(\lambda_{n-k}(K_T)\le a-cw_T\).
Finally, \(K=K_T+K_H\), where \(K_H\succeq0\) and
\(\operatorname{rank}(K_H)\le k\).  Rank-\(k\) Cauchy interlacing
\citep{HornJohnson2012MatrixAnalysis} gives
\(\lambda_n(K_T)\le\lambda_n(K)\le\lambda_{n-k}(K_T)\).
The wide Gaussian tail makes \(K_T\), and hence \(K\), positive definite with
probability one, so \(\lambda_n(K)=\widehat\mu_{\min}^+\).  Combining these
bounds proves
\(c w_T\le a-\widehat\mu_{\min}^+\le C w_T\).
The endpoint statement follows from
\(\nu<\widehat\mu_{\min}^+\).
\end{proof}

\begin{lemma}[Positive-filter common-spike floor]
\label{lem:journal-positive-head-floor}
Under Assumption~\ref{ass:journal-spike-flat}, with probability tending to one,
\[
R_{\rm GD+}^\star\wedge R_{\rm ridge+}^\star
\gtrsim
\Theta_H\left(\frac{a}{\lambda_h+a}\right)^2.
\]
\end{lemma}

\begin{proof}
Work on \(\mathcal E_{\rm cs}\) from
Lemma~\ref{lem:journal-spike-event}.  For either ordinary gradient flow or
nonnegative ridge, write the estimator in sample-space form with inverse
filter \(h_+(K)\).  Functional calculus gives
\[
0\preceq h_+(K)\preceq K^{-1}.
\tag{Positive inverse-filter cap}
\label{eq:cs-positive-inverse-filter-cap}
\]
Indeed, \(h_+(x)=(1-e^{-tx})/x\) for positive gradient flow and
\(h_+(x)=1/(x+\alpha)\) for nonnegative ridge.

Define the corresponding head signal map and its ridgeless cap by
\[
M_+:=\frac1nX_H^\top h_+(K)X_H,
\qquad
M_{\rm rid}:=\frac1nX_H^\top K^{-1}X_H .
\]
Congruence in \eqref{eq:cs-positive-inverse-filter-cap} yields
\(0\preceq M_+\preceq M_{\rm rid}\).  The matrix inversion lemma
\citep{HornJohnson2012MatrixAnalysis} gives
\[
M_{\rm rid}=C_0(I_k+C_0)^{-1},
\qquad
C_0:=\frac1nX_H^\top K_T^{-1}X_H .
\]
On \(\mathcal E_{\rm cs}\),
\(\|K_T-aI_n\|_{\rm op}=O(w_T)=o(a)\) and
\(G_H=\lambda_hI_k+O(\omega_{H,n}w_T)\).  The resolvent identity therefore
implies, uniformly in operator norm,
\[
C_0=\frac{\lambda_h}{a}I_k+o(1),
\qquad
M_{\rm rid}
=
\frac{\lambda_h}{a+\lambda_h}I_k+o(1).
\]

Let \(\beta=\beta_H^\star\) and
\(q=\lambda_h/(a+\lambda_h)\).  The Loewner inequality above gives
\[
\beta^\top(I_k-M_+)\beta
\ge
\beta^\top(I_k-M_{\rm rid})\beta
\ge
\{1-q-o(1)\}\|\beta\|_2^2 .
\]
No squared Loewner order or commutativity is used.  Cauchy--Schwarz now gives
\[
\|(I_k-M_+)\beta\|_2^2
\ge
\frac{\{\beta^\top(I_k-M_+)\beta\}^2}{\|\beta\|_2^2}
\ge
\left\{\frac{a}{a+\lambda_h}-o(1)\right\}^2\|\beta\|_2^2 .
\]
The population prediction bias is at least the head contribution
\(\lambda_h\|(I_k-M_+)\beta_H^\star\|_2^2\); the tail bias and variance are
nonnegative.  Since
\(\Theta_H=\lambda_h\|\beta_H^\star\|_2^2\) and
\(a/(\lambda_h+a)\asymp1\), the claimed lower bound follows uniformly over
both positive-filter classes.
\end{proof}

\begin{lemma}[Explicit supercritical common-spike upper bound]
\label{lem:journal-signed-upper}
Under Assumption~\ref{ass:journal-spike-flat}, with probability tending to one,
\[
R_{\rm sign}(\nu_\star,t_\star)
\vee
R_{\rm sign}^{\rm disc}(\nu_\star,\eta_\star,m)
\lesssim
\frac{\sigma_\varepsilon^2k}{n}
+
(\Theta_H+\sigma_\varepsilon^2)
\frac{a\lambda_T}{\lambda_h^2}.
\]
\end{lemma}

\begin{proof}
Work on \(\mathcal E_{\rm cs}\).  The head-map estimate
\eqref{eq:common-spike-head-map-event} and the identity
\(\lambda_hh_{\nu_\star,t_\star}(a+\lambda_h)=1\) give
\[
B_{\rm head}(\nu_\star,t_\star)
\lesssim
\Theta_Hr_T^2
\asymp
\Theta_H\frac{a\lambda_T}{\lambda_h^2}.
\]
The \(k\) resolved head coordinates contribute at most
\(C\sigma_\varepsilon^2k/n\) variance.

For every \(\mu\in\operatorname{spec}(K)\), the tail event and finite-rank
interlacing give \(\mu\ge a-Cw_T\).  Consequently,
\[
h_{\nu_\star,t_\star}(\mu)
=
\int_0^{1/\lambda_h}e^{-(\mu-a-\lambda_h)s}\,ds
\le
\frac{C}{\lambda_h},
\]
because \(w_T/\lambda_h\to0\).  Since
\(K_{2,T}=\lambda_TK_T\) and
\(\operatorname{tr}(K_T)\asymp na\), the conditional tail-noise contribution
is bounded by
\[
\frac{\sigma_\varepsilon^2}{n}
\operatorname{tr}\!\left\{
h_{\nu_\star,t_\star}(K)^2K_{2,T}
\right\}
\lesssim
\frac{\sigma_\varepsilon^2a\lambda_T}{\lambda_h^2}.
\]

Head support removes true tail-signal bias, but sample-induced head-to-tail
leakage remains.  Here it is controlled directly in the common-spike model.
Indeed,
\[
\mathcal L_{H\to T}
=
\left\|
\frac1nX_T^\top
h_{\nu_\star,t_\star}(K)X_H\beta_H^\star
\right\|_{\Sigma_T}^2
\le
\left\|\frac{X_H\beta_H^\star}{\sqrt n}\right\|_2^2
\left\|h_{\nu_\star,t_\star}(K)\right\|_{\rm op}^2
\left\|K_{2,T}\right\|_{\rm op}.
\]
On \(\mathcal E_{\rm cs}\), the three factors are bounded by
\(C\Theta_H\), \(C/\lambda_h^2\), and \(Ca\lambda_T\), respectively.
Consequently,
\[
\mathcal L_{H\to T}
\lesssim
\Theta_H\frac{a\lambda_T}{\lambda_h^2}.
\]
Adding head bias, head variance, tail noise, and leakage proves the claim.
For the discrete path, the geometric-series identity gives, uniformly in
\(m\ge1\),
\[
\left|
h_{\nu_\star,\eta_\star,m}^{\rm disc}(\mu)
\right|
\le
\frac{C}{\lambda_h}
\qquad
(\mu\ge a-Cw_T).
\]
Together with the discrete head-map part of
\eqref{eq:common-spike-head-map-event}, the same noise and leakage calculation
proves the identical bound for
\(R_{\rm sign}^{\rm disc}(\nu_\star,\eta_\star,m)\).
\end{proof}

\begin{lemma}[Admissible endpoint lower bound]
\label{lem:journal-negative-lower}
Under Assumption~\ref{ass:journal-spike-flat}, suppose also that
\(n\lambda_T/\lambda_h\to\infty\).  Then, with probability tending to one,
\[
R_{\rm neg}^\star
\gtrsim
\frac{\sigma_\varepsilon^2k}{n}
+
\frac{\sigma_\varepsilon\sqrt{\Theta_Ha\lambda_T}}{\lambda_h}.
\]
\end{lemma}

\begin{proof}
Work on \(\mathcal E_{\rm cs}\) and write \(b=a-\nu\).
Proposition~\ref{prop:journal-mp-barrier} implies \(b\gtrsim w_T\) for every
admissible endpoint.  Define the balance scale
\[
b_\star
:=
\left(
\frac{\sigma_\varepsilon^2a\lambda_T\lambda_h^2}{\Theta_H}
\right)^{1/4},
\]
for which \(b_\star/w_T\asymp(\lambda_h/\lambda_T)^{1/4}\to\infty\).

On the tail sub-bulk in
\eqref{eq:common-spike-projected-tail-mass},
\[
0<\mu_i-\nu=b+(\mu_i-a)\lesssim b+w_T\lesssim b .
\]
The projected squared-mass readout consequently gives
\[
V_{\rm tail}(A_\nu)
\gtrsim
\frac{\sigma_\varepsilon^2a\lambda_T}{b^2}.
\]
Therefore, if \(b\le Cb_\star\), the tail variance alone is at least a
constant multiple of
\[
\frac{\sigma_\varepsilon\sqrt{\Theta_Ha\lambda_T}}{\lambda_h}.
\tag{Small-\(b\) endpoint regime}
\label{eq:small-b-endpoint-regime}
\]

It remains to control \(b\ge cb_\star\).  In this regime \(b/w_T\to\infty\),
so \(B_\nu:=K_T-\nu I_n=(K_T-aI_n)+bI_n\)
is positive definite.  Put
\(C_\nu:=n^{-1}X_H^\top B_\nu^{-1}X_H\).
The matrix inversion lemma \citep{HornJohnson2012MatrixAnalysis} gives the
exact head signal map
\[
\frac1nX_H^\top(K-\nu I_n)^{-1}X_H
=
C_\nu(I_k+C_\nu)^{-1}.
\]
The tail-floor and head-Wishart events imply
\[
C_\nu
=
\frac{\lambda_h}{b}I_k
+
E_\nu,
\qquad
\|E_\nu\|_{\rm op}
\lesssim
\frac{\lambda_hw_T}{b^2}
+
\frac{\lambda_h}{b\sqrt n}.
\]
The map \(C\mapsto C(I+C)^{-1}\) damps this error by
\((1+\lambda_h/b)^{-2}\).  Uniformly for \(b\ge cb_\star\),
\[
\left\|
C_\nu(I+C_\nu)^{-1}
-
\frac{\lambda_h}{\lambda_h+b}I_k
\right\|_{\rm op}
=
o\!\left(\frac{b}{\lambda_h+b}\right),
\]
where \(w_T/b_\star\to0\) and
\(n\lambda_T/\lambda_h\to\infty\) control the two error terms.  Hence
\[
B_{\rm head}(A_\nu)
\gtrsim
\Theta_H
\left(\frac{b}{\lambda_h+b}\right)^2
\gtrsim
\frac{\sigma_\varepsilon\sqrt{\Theta_Ha\lambda_T}}{\lambda_h}
\tag{Large-\(b\) endpoint regime}
\label{eq:large-b-endpoint-regime}
\]
for \(b\ge cb_\star\); when \(b\gtrsim\lambda_h\), the constant head-bias
floor is even larger.

For completeness, put \(D_\nu=K-\nu I_n\).  Since \(D_\nu\succ0\) for every
admissible endpoint and \(K_{2,H}=\lambda_hX_HX_H^\top/n\), its resolved-head
variance satisfies
\[
\frac{\sigma_\varepsilon^2}{n}
\tr(D_\nu^{-1}K_{2,H}D_\nu^{-1})
=
\frac{\sigma_\varepsilon^2\lambda_h}{n^2}
\|D_\nu^{-1}X_H\|_F^2
\ge
\frac{\sigma_\varepsilon^2\lambda_h\tr(G_H)}
{n\|D_\nu\|_{\rm op}^2}
\gtrsim
\frac{\sigma_\varepsilon^2k}{n}.
\]
Here \(\|D_\nu\|_{\rm op}\le\|K\|_{\rm op}\lesssim\lambda_h\) and
\(\tr(G_H)\asymp k\lambda_h\) on \(\mathcal E_{\rm cs}\), so the bound is
uniform in \(\nu\).
Combining
\eqref{eq:small-b-endpoint-regime} and
\eqref{eq:large-b-endpoint-regime} proves the uniform lower envelope.
\end{proof}

\begin{proof}[Proof of Theorem~\ref{thm:journal-spike-flat-separation}]
The three risk bounds follow from
Lemmas~\ref{lem:journal-positive-head-floor}--\ref{lem:journal-negative-lower}.
The non-saturation condition makes the common head-variance term lower order
than the signed tail-exposure scale.  In the interior regime, division by the
endpoint and positive-filter lower bounds yields the two ratios.
\end{proof}

\subsection{Case II: heterogeneous-head, general-tail recovery}
\label{app:journal-multispike-floor-critical-proofs}

Write
\[
X=[X_H\;X_T],\qquad
K_H=\frac1nX_HX_H^\top,\qquad
K_T=\frac1nX_TX_T^\top,\qquad
K=K_H+K_T,
\]
and define the centered first-moment tail Gram and the squared-spectrum
readout
\[
E_T:=K_T-aI_n,\qquad
K_{2,T}:=\frac1nX_T\Lambda_TX_T^\top
=\frac1n\sum_{j>k}\lambda_j^2z_jz_j^\top .
\]
Let \(P_H\) project onto \(\operatorname{range}(X_H)\).  With probability
tending to one \(X_H\) has full column rank and
\(P_H=n^{-1}X_HG_H^{-1}X_H^\top\) with
\(G_H=n^{-1}X_H^\top X_H\).

\begin{lemma}[Weighted Gaussian Gram concentration]
\label{lem:weighted-gaussian-gram-concentration}
Let \(z_1,\ldots,z_d\stackrel{\rm iid}{\sim}N(0,I_n)\), let
\(\alpha_1,\ldots,\alpha_d\ge0\) be deterministic, and put
\[
W_\alpha
:=
\frac1n\sum_{j=1}^d
\alpha_j(z_jz_j^\top-I_n),
\qquad
A_2:=\sum_{j=1}^d\alpha_j^2,
\qquad
\alpha_{\max}:=\max_{1\le j\le d}\alpha_j .
\]
There is a numerical constant \(C\) such that, for every \(s>0\),
\[
\mathbb P\!\left[
\|W_\alpha\|_{\rm op}
>
C\left\{
\sqrt{\frac{A_2(n+s)}{n^2}}
+
\frac{\alpha_{\max}(n+s)}{n}
\right\}
\right]
\le 2e^{-s}.
\tag{Weighted Gaussian Gram bound}
\label{eq:weighted-gaussian-gram-bound}
\]
If \(U\in\mathbb R^{n\times q}\) has orthonormal columns and is deterministic
or independent of \((z_j)_{j=1}^d\), then, conditionally on \(U\),
\[
\mathbb P\!\left[
\|U^\top W_\alpha U\|_{\rm op}
>
C\left\{
\sqrt{\frac{A_2(q+s)}{n^2}}
+
\frac{\alpha_{\max}(q+s)}{n}
\right\}
\ \middle|\ U
\right]
\le 2e^{-s}.
\tag{Fixed-frame weighted Gaussian Gram bound}
\label{eq:fixed-frame-weighted-gaussian-gram-bound}
\]
\end{lemma}

\begin{proof}
For independent \(g_j\sim N(0,1)\), the exponential-moment identity
\[
\log\mathbb E
\exp\!\left\{
\theta\sum_{j=1}^d\alpha_j(g_j^2-1)
\right\}
=
\sum_{j=1}^d
\left\{
-\theta\alpha_j-\frac12\log(1-2\theta\alpha_j)
\right\}
\]
holds for \(0<\theta<(2\alpha_{\max})^{-1}\).  The inequality
\(-x-\log(1-x)\le x^2/\{2(1-x)\}\), \(0\le x<1\), followed by Chernoff
optimization gives, for every \(x>0\),
\[
\mathbb P\!\left(
\left|
\sum_{j=1}^d\alpha_j(g_j^2-1)
\right|
>
2\sqrt{A_2x}+2\alpha_{\max}x
\right)
\le 2e^{-x}.
\tag{Weighted chi-square bound}
\label{eq:weighted-chi-square-bound}
\]
This is the weighted chi-square inequality of
\citet[Lemma~1]{LaurentMassart2000}; the calculation is included to record
the constants and the dependence on \(A_2\) and \(\alpha_{\max}\).
The lower tail follows from the same calculation with \(\theta<0\), and its
bound \(2\sqrt{A_2x}\) is dominated by the displayed two-sided threshold.

Let \(\mathcal N_n\) be a \(1/4\)-net of the unit sphere in
\(\mathbb R^n\) with \(|\mathcal N_n|\le9^n\); see, for example,
\citet[Section~4.4]{Vershynin2018High-dimensional}.  For every symmetric matrix
\(M\),
\[
\|M\|_{\rm op}
\le
2\max_{u\in\mathcal N_n}|u^\top Mu|;
\]
indeed, approximate a unit vector attaining the operator norm by an element
of the net and absorb the resulting \(2(1/4)\|M\|_{\rm op}\) error.
For fixed \(u\), the variables \(u^\top z_j\) are independent \(N(0,1)\).
Apply \eqref{eq:weighted-chi-square-bound}, take a union bound over
\(\mathcal N_n\), and set \(x=n\log9+s\).  Absorbing \(\log9\) and the net
factor into a numerical constant proves
\eqref{eq:weighted-gaussian-gram-bound}.

Conditionally on \(U\), the vectors \(U^\top z_j\) are independent
\(N(0,I_q)\).  Repeating the same argument on a \(1/4\)-net of the unit
sphere in \(\mathbb R^q\), now with \(x=q\log9+s\), proves the asserted
fixed-frame inequality.
\end{proof}

\begin{lemma}[Gaussian Tsigler--Bartlett localization event]
\label{lem:journal-multispike-gaussian-event}
Under Assumption~\ref{ass:journal-multispike}, there is an event
\(\mathcal E_{\rm ht}\) with probability tending to one on which
\[
\left\|
\Lambda_H^{-1/2}G_H\Lambda_H^{-1/2}-I_k
\right\|_{\rm op}
\lesssim\chi_{n,k},
\]
\[
\begin{aligned}
\|E_T\|_{\rm op}&\lesssim w_k,&
\|P_HE_TP_H\|_{\rm op}
&\lesssim w_k\chi_{n,k},&
\|P_HE_TP_H^\perp\|_{\rm op}&\lesssim w_k,\\
\|K_{2,T}\|_{\rm op}&\lesssim w_k^2,&
\tr(K_{2,T})&=(1+o(1))S_k.&
\end{aligned}
\]
At least \(n-k\) eigenvalues of \(K\) lie in
\([a-Cw_k,a+Cw_k]\).  If \(P_B\) denotes their spectral projector, then
\[
\tr(P_BK_{2,T})\ge cS_k .
\tag{Projected square mass}
\label{eq:ms-projected-square-mass}
\]
\end{lemma}

\begin{proof}
In the population eigenbasis,
\[
K_T-aI_n
=
\frac1n\sum_{j>k}\lambda_j(z_jz_j^\top-I_n).
\]
Apply Lemma~\ref{lem:weighted-gaussian-gram-concentration} with
\(\alpha_j=\lambda_j\), \(A_2=S_k\),
\(\alpha_{\max}=\lambda_{k+1}\), and \(s=n\).  With probability at least
\(1-2e^{-n}\),
\[
\|K_T-aI_n\|_{\rm op}
\lesssim
\lambda_{k+1}+\sqrt{\frac{S_k}{n}}
=w_k .
\]
This is the same tail-Gram scale used in
the Tsigler--Bartlett conditioning argument: after division by
\(a=T_k/n\), the two terms are \(n/\mathfrak r_k\) and
\(\sqrt{n/\mathfrak R_k}\).  Conditional on \(X_H\), \(P_H\) is a fixed
rank-\(k\) projector independent of the tail columns.  Applying the weighted
Gram bound \eqref{eq:fixed-frame-weighted-gaussian-gram-bound} conditionally
on \(X_H\), with \(q=k\) and \(s=3\log n\), we obtain, with conditional
failure probability at most \(2n^{-3}\),
\[
\|P_HE_TP_H\|_{\rm op}
\lesssim
\frac{\sqrt{S_k(k+\log n)}}{n}
+
\frac{\lambda_{k+1}(k+\log n)}{n}
\lesssim
w_k\chi_{n,k},
\]
where the last inequality uses \(\chi_{n,k}=o(1)\), which follows from
Assumption~\ref{ass:journal-multispike}.  The off-diagonal estimate
follows from
\(\|P_HE_TP_H^\perp\|_{\rm op}\le\|E_T\|_{\rm op}\).

For the squared-spectrum readout, apply
\eqref{eq:weighted-gaussian-gram-bound} with
\(\alpha_j=\lambda_j^2\).  Since
\[
\sum_{j>k}\lambda_j^4
\le
\lambda_{k+1}^2S_k,
\qquad
\lambda_{k+1}\sqrt{\frac{S_k}{n}}
\le
\frac12\left(\lambda_{k+1}^2+\frac{S_k}{n}\right),
\]
the same \(s=n\) choice gives
\[
\begin{aligned}
\|K_{2,T}\|_{\rm op}
&\le
\frac{S_k}{n}
+
\left\|
K_{2,T}-\frac{S_k}{n}I_n
\right\|_{\rm op}\\
&\lesssim
\frac{S_k}{n}
+
\lambda_{k+1}^2
+
\sqrt{\frac{\sum_{j>k}\lambda_j^4}{n}}
\lesssim
\lambda_{k+1}^2+\frac{S_k}{n}
\lesssim w_k^2.
\end{aligned}
\]
Moreover,
\[
\tr(K_{2,T})
=
\frac1n\sum_{j>k}\lambda_j^2\|z_j\|_2^2
\]
has mean \(S_k\) and variance
\[
\operatorname{Var}\{\tr(K_{2,T})\}
=
\frac{2}{n}\sum_{j>k}\lambda_j^4
\le
\frac{2\lambda_{k+1}^2S_k}{n}.
\]
Therefore
\(\tr(K_{2,T})=(1+o_{\mathbb P}(1))S_k\) by Chebyshev's inequality and
\(S_k/\lambda_{k+1}^2\to\infty\), which follows from
Assumption~\ref{ass:journal-multispike}.

Finally,
\[
\Lambda_H^{-1/2}G_H\Lambda_H^{-1/2}
=
\frac1n Z_H^\top Z_H,
\]
where \(Z_H\in\mathbb R^{n\times k}\) has independent standard Gaussian
entries.  Apply the weighted chi-square bound above on a \(1/4\)-net of the
unit sphere in \(\mathbb R^k\), with \(n\) unit weights and
\(x=k\log9+3\log n\), to obtain
\[
\left\|
\Lambda_H^{-1/2}G_H\Lambda_H^{-1/2}-I_k
\right\|_{\rm op}
\lesssim
\chi_{n,k}
\]
with failure probability \(O(n^{-3})\).

On the intersection of these events, all eigenvalues of \(K_T\) lie in
\([a-Cw_k,a+Cw_k]\).  Since \(K_H\succeq0\) and
\(\operatorname{rank}(K_H)\le k\), rank-\(k\) Cauchy interlacing
\citep{HornJohnson2012MatrixAnalysis} shows that
at most \(k\) eigenvalues of \(K=K_T+K_H\) can lie above this band and none
can lie below it.  Thus the spectral projector \(P_B\) onto the band has
rank at least \(n-k\).  Since \(I-P_B\) has rank at most \(k\) and
\(K_{2,T}\succeq0\),
\[
\tr(P_BK_{2,T})
\ge
\tr(K_{2,T})-k\|K_{2,T}\|_{\rm op},
\]
and
\[
k\|K_{2,T}\|_{\rm op}
\lesssim k w_k^2
=o(S_k)
\]
by Assumption~\ref{ass:journal-multispike}.
This proves \eqref{eq:ms-projected-square-mass}.
\end{proof}

\begin{lemma}[Invariant-graph localized Duhamel head sandwich]
\label{lem:localized-duhamel-head-sandwich}
Let \(A_0\) and \(E\) be self-adjoint, let \(P\) be a rank-\(k\)
orthogonal projector, and suppose that
\[
A_0P=PA_0,\qquad A_0P^\perp=0,\qquad
\operatorname{spec}(A_0|_{\operatorname{range}(P)})\subset[g,G]
\]
for \(0<g\le G\) and \(G/g\le\kappa\).  There is a constant
\(c_\kappa>0\), depending only on \(\kappa\), such that the following
holds.  Let \(A=A_0+E\), with
\(\|E\|_{\rm op}\le c_\kappa g\), and put
\(\epsilon_H=\|PEP\|_{\rm op}\),
\(\epsilon_{HT}=\|PEP^\perp\|_{\rm op}\).  For
\(h_t(A)=\int_0^te^{-sA}\,ds\) and every \(t\ge0\),
\[
\left\|P\{h_t(A)-h_t(A_0)\}P\right\|_{\rm op}
\lesssim
\frac{\epsilon_H}{g^2}
+
\frac{\epsilon_{HT}^2}{g^3}
\left\{1+gt\,e^{C_\kappa t\|E\|_{\rm op}}\right\}.
\tag{Invariant-graph localized Duhamel}
\label{eq:localized-duhamel-head}
\]
Moreover,
\[
\|h_t(A)P\|_{\rm op}
\le
\frac{C_\kappa}{g}
\left\{
1+\epsilon_{HT}t\,e^{C_\kappa t\|E\|_{\rm op}}
\right\}.
\tag{Localized head exposure}
\label{eq:localized-head-exposure}
\]
In particular, if \(t\|E\|_{\rm op}\le c_0\), then
\(\|h_t(A)P\|_{\rm op}\le C_{\kappa,c_0}g^{-1}\).
All constants are independent of the two block dimensions and of \(t\).
\end{lemma}

\begin{proof}
Invariant graph subspaces and their operator Riccati equations are standard
in spectral perturbation theory
\citep{KostrykinMakarovMotovilov2005}.  We give the full finite-dimensional
construction because the dimension-free quadratic leave-and-return bound and
the noncontractive exposure factor in the lemma are the new ingredients.
Relative to \(P\oplus P^\perp\), write
\[
A_0=
\begin{pmatrix}
H&0\\
0&0
\end{pmatrix},
\qquad
E=
\begin{pmatrix}
E_{HH}&E_{HT}\\
E_{TH}&E_{TT}
\end{pmatrix},
\qquad
\operatorname{spec}(H)\subset[g,G].
\]
Set
\[
H_1:=H+E_{HH},\qquad
B:=E_{HT},\qquad
C_T:=E_{TT},\qquad
\delta:=\|E\|_{\rm op},\qquad b:=\|B\|_{\rm op}.
\]
After decreasing \(c_\kappa\), if necessary,
\[
H_1\succeq\frac{3g}{4}I,\qquad
\lambda_{\max}(C_T)\le\frac g4,\qquad
\Delta:=
\lambda_{\min}(H_1)-\lambda_{\max}(C_T)
\ge\frac g2 .
\tag{Separated realized blocks}
\label{eq:ms-separated-realized-blocks}
\]

We first construct the perturbed head subspace.  For a map
\(Z:\operatorname{range}(P)\to\operatorname{range}(P^\perp)\), define
\(\mathcal S(Z):=ZH_1-C_TZ\).
The separation in \eqref{eq:ms-separated-realized-blocks} gives the explicit
inverse
\[
\mathcal S^{-1}(Y)
=
\int_0^\infty e^{sC_T}Ye^{-sH_1}\,ds,
\qquad
\|\mathcal S^{-1}\|_{\rm op\to op}\le\Delta^{-1}.
\]
Indeed, the integral is norm-convergent, and differentiating its integrand
shows directly that applying \(\mathcal S\) returns \(Y\).

The graph of \(Z\) is invariant under \(A\) precisely when
\[
ZH_1-C_TZ=B^\top-ZBZ.
\]
Consider
\(\Phi(Z)=\mathcal S^{-1}(B^\top-ZBZ)\) on the closed ball
\(\|Z\|_{\rm op}\le2b/\Delta\).  On this ball,
\[
\|\Phi(Z)\|_{\rm op}
\le
\frac b\Delta
\left\{1+\frac{4b^2}{\Delta^2}\right\}
<\frac{2b}{\Delta},
\]
and, for two points in the ball,
\[
\|\Phi(Z)-\Phi(Z')\|_{\rm op}
\le
\frac{4b^2}{\Delta^2}\|Z-Z'\|_{\rm op}.
\]
Choosing \(c_\kappa\) sufficiently small makes the last factor strictly less
than one.  The contraction theorem therefore yields a unique solution in the
ball, with
\[
\|Z\|_{\rm op}\le\frac{2b}{\Delta}\le C\frac bg.
\tag{Graph-angle bound}
\label{eq:ms-graph-angle-bound}
\]

Put
\[
C_Z=(I+Z^\top Z)^{-1/2},\qquad
D_Z=(I+ZZ^\top)^{-1/2},
\]
and define the orthogonal graph rotation
\[
Q=
\begin{pmatrix}
C_Z&-Z^\top D_Z\\
ZC_Z&D_Z
\end{pmatrix}.
\]
The singular-value decomposition of \(Z\) gives
\(D_ZZ=ZC_Z\) and \(C_ZZ^\top=Z^\top D_Z\); direct multiplication then
verifies \(Q^\top Q=I\).
The first block column spans the invariant graph, and self-adjointness makes
its orthogonal complement invariant.  Hence
\[
Q^\top AQ=
\begin{pmatrix}
\widetilde H&0\\
0&\widetilde C_T
\end{pmatrix}.
\]
Scalar functional calculus applied to
\((1+x)^{-1/2}\) gives
\[
\|C_Z-I\|_{\rm op}+\|D_Z-I\|_{\rm op}
\le C\|Z\|_{\rm op}^2.
\]
Moreover, direct multiplication gives
\[
\begin{aligned}
\widetilde H
&=
C_Z
\{H_1+BZ+Z^\top B^\top+Z^\top C_TZ\}
C_Z,\\
\widetilde C_T
&=
D_Z
\{C_T-ZB-B^\top Z^\top+ZH_1Z^\top\}
D_Z .
\end{aligned}
\]
Using \eqref{eq:ms-graph-angle-bound},
\(\|H_1\|_{\rm op}\le(\kappa+c_\kappa)g\), and
\(\|C_T\|_{\rm op}\le\delta\), we obtain
\[
\|\widetilde H-H_1\|_{\rm op}
+
\|\widetilde C_T-C_T\|_{\rm op}
\le
C_\kappa\frac{b^2}{g}.
\tag{Quadratic block relocation}
\label{eq:ms-quadratic-block-relocation}
\]
After one further decrease of \(c_\kappa\),
\[
\widetilde H\succeq\frac g2I,
\qquad
\widetilde C_T\succeq-d_\star I,
\qquad
d_\star:=\delta+C_\kappa b^2/g
\le C_\kappa\delta.
\]

Functional calculus in the rotated basis now gives the exact compression
\[
Ph_t(A)P
=
C_Zh_t(\widetilde H)C_Z
+
Z^\top D_Zh_t(\widetilde C_T)D_ZZ.
\tag{Rotated head compression}
\label{eq:ms-rotated-head-compression}
\]
For the positive upper block, the standard Duhamel identity for matrix
exponentials \citep{Higham2008Functions} and
\eqref{eq:ms-quadratic-block-relocation} yield
\[
\|h_t(\widetilde H)-h_t(H_1)\|_{\rm op}
\le
\|\widetilde H-H_1\|_{\rm op}
\int_0^\infty se^{-gs/2}\,ds
\le
C_\kappa\frac{b^2}{g^3}.
\]
The same calculation applied to \(H_1-H=E_{HH}\) gives
\[
\|h_t(H_1)-h_t(H)\|_{\rm op}
\le C_\kappa\frac{\epsilon_H}{g^2}.
\]
The diagonal rotation in
\eqref{eq:ms-rotated-head-compression} costs at most
\(C_\kappa b^2/g^3\), because
\(\|h_t(\widetilde H)\|_{\rm op}\le2/g\).  The lower-cluster contribution is
bounded by
\[
\|Z\|_{\rm op}^2
\|h_t(\widetilde C_T)\|_{\rm op}
\le
C_\kappa\frac{b^2}{g^2}\,
t e^{d_\star t}.
\]
Combining the last four displays and using \(d_\star\le C_\kappa\delta\)
proves the claimed head bound. This is
\eqref{eq:localized-duhamel-head}.

Finally,
\[
Q^\top P\big|_{\operatorname{range}(P)}=
\begin{pmatrix}
C_Z\\
-D_ZZ
\end{pmatrix},
\]
so the same block functional calculus gives
\[
\|h_t(A)P\|_{\rm op}
\le
\|h_t(\widetilde H)\|_{\rm op}
+
\|Z\|_{\rm op}\|h_t(\widetilde C_T)\|_{\rm op}
\le
\frac{C_\kappa}{g}
\left\{
1+bt\,e^{C_\kappa\delta t}
\right\}.
\]
Since \(b=\epsilon_{HT}\), this proves
\eqref{eq:localized-head-exposure} and its exposure-capped consequence.
\end{proof}

\begin{lemma}[General-tail floor-critical signed upper bound]
\label{lem:journal-multispike-signed-upper}
Under Assumption~\ref{ass:journal-multispike}, let
\(t\ge0\) and suppose \(w_kt=o(1)\).  Then, with probability tending to one,
\[
\begin{aligned}
\Risk_{\rm pop}(f_{a,t}\mid X)
\lesssim{}&
\Theta_He^{-2(1-o(1))\lambda_-t}
+
\Theta_H\delta_H(t)^2
+
\frac{\sigma_\varepsilon^2k}{n}
+
\Theta_H\frac{w_k^2}{\lambda_-^2}
+
\frac{\sigma_\varepsilon^2S_k}{n}
t^2e^{Cw_kt},
\end{aligned}
\tag{General-tail signed-risk envelope}
\label{eq:ms-localized-signed-envelope}
\]
where
\[
\delta_H(t)
\lesssim
\left\{
\frac{w_k}{\lambda_-}
\chi_{n,k}
+
\frac{w_k^2}{\lambda_-^2}
\left(1+\lambda_-t\,e^{Cw_kt}\right)
\right\}.
\tag{Localized head-map error}
\label{eq:ms-localized-head-map-error}
\]
The implicit constants may depend only on the fixed upper bound for
\(\lambda_+/\lambda_-\).
\end{lemma}

\begin{proof}
Set \(A=K-aI_n=K_H+E_T\) and \(A_0=K_H\).  On
\(\mathcal E_{\rm ht}\), the nonzero eigenvalues of \(A_0\) are those of
\(G_H\) and lie in
\([(1-o(1))\lambda_-,(1+o(1))\lambda_+]\).  Take
\[
g:=\lambda_{\min}(G_H),\qquad
\epsilon_H:=\|P_HE_TP_H\|_{\rm op},\qquad
\epsilon_{HT}:=\|P_HE_TP_H^\perp\|_{\rm op}.
\]
Because \(w_k=o(\lambda_-)\), the event bounds imply, for all sufficiently
large \(n\),
\[
\|E_T\|_{\rm op}\le c_\kappa g,\qquad
\epsilon_H\lesssim w_k\chi_{n,k},\qquad
\epsilon_{HT}\lesssim w_k,\qquad
t\|E_T\|_{\rm op}=o(1).
\tag{Localized-lemma hypotheses}
\label{eq:ms-localized-lemma-hypotheses}
\]
Thus every hypothesis of
Lemma~\ref{lem:localized-duhamel-head-sandwich} is satisfied with \(P=P_H\).
Since \(P_HX_H=X_H\) and
\(\|X_H/\sqrt n\|_{\rm op}^2=\|G_H\|_{\rm op}\lesssim\lambda_+\), the lemma
gives
\[
\left\|
\frac1nX_H^\top
\{h_{a,t}(K)-h_{a,t}(K_H+aI_n)\}X_H
\right\|_{\rm op}
\le\delta_H(t).
\]
Substitution of these event bounds into the localized Duhamel estimate gives
the asserted head-map error.
In the floor model,
\[
\frac1nX_H^\top h_{a,t}(K_H+aI_n)X_H=I_k-e^{-tG_H},
\]
which yields the first two terms in
\eqref{eq:ms-localized-signed-envelope}.

For the remaining terms put
\(K_2:=n^{-1}X\Sigma X^\top=K_{2,H}+K_{2,T}\).
The exact conditional Gaussian-risk decomposition is
\[
\begin{aligned}
\Risk_{\rm pop}(f_{a,t}\mid X)
={}&
\left\|
\Lambda_H^{1/2}
\left\{
\frac1nX_H^\top h_{a,t}(K)X_H-I_k
\right\}\beta_H^\star
\right\|_2^2\\
&+
\frac1n(X_H\beta_H^\star)^\top
h_{a,t}(K)K_{2,T}h_{a,t}(K)(X_H\beta_H^\star)\\
&+
\frac{\sigma_\varepsilon^2}{n}
\tr\{K_2h_{a,t}(K)^2\}.
\end{aligned}
\tag{Exact heterogeneous-tail risk}
\label{eq:ms-exact-risk}
\]
On \(\mathcal E_{\rm ht}\),
\[
\|h_{a,t}(K)\|_{\rm op}\le te^{Cw_kt},
\qquad
K_{2,T}\preceq Cw_k^2I_n.
\]
The global bound and
\[
\tr\{K_{2,T}h_{a,t}(K)^2\}
\le
\|h_{a,t}(K)\|_{\rm op}^2\tr(K_{2,T})
\]
give the tail-noise term in
\eqref{eq:ms-localized-signed-envelope}, because
\(\tr(K_{2,T})=(1+o(1))S_k\) on
\(\mathcal E_{\rm ht}\).
For the tail-signal leakage, the exposure conclusion of
Lemma~\ref{lem:localized-duhamel-head-sandwich} and
\eqref{eq:ms-localized-lemma-hypotheses} give
\[
\|h_{a,t}(K)P_H\|_{\rm op}\lesssim\lambda_-^{-1}.
\]
Consequently, the second line of \eqref{eq:ms-exact-risk} is at most
\[
\frac{\|X_H\beta_H^\star\|_2^2}{n}
\|h_{a,t}(K)P_H\|_{\rm op}^2
\|K_{2,T}\|_{\rm op}
\lesssim
\Theta_H\frac{w_k^2}{\lambda_-^2},
\]
which is the displayed tail-signal leakage term.
Finally, \(K_{2,H}\) has rank at most \(k\) and
\(\|K_{2,H}\|_{\rm op}\lesssim\lambda_+^2\), so
\[
\frac{\sigma_\varepsilon^2}{n}
\tr\{K_{2,H}h_{a,t}(K)^2\}
\lesssim
\frac{\sigma_\varepsilon^2k}{n}
\frac{\lambda_+^2}{\lambda_-^2}
\lesssim
\frac{\sigma_\varepsilon^2k}{n}.
\]
Together with the head-bias bound, these estimates prove
\eqref{eq:ms-localized-signed-envelope}.
\end{proof}

\begin{lemma}[General-tail comparator lower bounds]
\label{lem:journal-multispike-comparators}
Under Assumption~\ref{ass:journal-multispike}, with probability tending to
one,
\[
R_{\rm scale}^\star
\ge\Delta_{\rm shape}-o(1)\Theta_H,
\qquad
R_+^\star
\gtrsim
\sum_{j=1}^k\Theta_j
\left(\frac{a}{a+\lambda_j}\right)^2,
\]
and \(R_{\rm neg}^\star\) satisfies
\eqref{eq:multispike-endpoint-envelope}.
\end{lemma}

\begin{proof}
The matrix inversion lemma \citep{HornJohnson2012MatrixAnalysis} gives the
ridgeless head map
\[
M_{\rm rid}
:=
\frac1nX_H^\top K^{-1}X_H
=C_0(I_k+C_0)^{-1},
\qquad
C_0:=\frac1nX_H^\top K_T^{-1}X_H.
\]
Because \(\|K_T-aI_n\|_{\rm op}=o(a)\) and
\(G_H=\Lambda_H+o_{\mathbb P}(\lambda_-)\),
\[
M_{\rm rid}
=
\operatorname{diag}\!\left(
\frac{\lambda_1}{a+\lambda_1},\ldots,
\frac{\lambda_k}{a+\lambda_k}
\right)+o_{\mathbb P}(1).
\]
The head bias of \(c\widehat\beta_{\rm rid}\) therefore converges uniformly
on bounded \(c\)-sets to
\(\sum_j\Theta_j(1-cq_j)^2\); outside a large bounded set it is larger.
This proves the scale lower bound.

For positive filters, every ordinary early-stopped gradient flow and
nonnegative-ridge map has scalar filter in \([0,1]\), hence
\(0\preceq h_+(K)\preceq K^{-1}\); by congruence the realized head map
\(M_+=\tfrac1nX_H^\top h_+(K)X_H\) satisfies
\(0\preceq M_+\preceq M_{\rm rid}\), so \(I_k-M_+\succeq I_k-M_{\rm rid}\).
Under the shared framework \(\lambda_+/\lambda_-=O(1)\) and \(a\asymp\lambda_-\),
so \(I_k-M_{\rm rid}=\operatorname{diag}\!\big(a/(a+\lambda_j)\big)+o_{\mathbb
P}(1)\succeq c_0I_k\) for a constant \(c_0>0\).  With
\(\Lambda_H\succeq\lambda_-I_k\) and
\(\Theta_H\le\lambda_+\|\beta_H^\star\|_2^2\), directional Cauchy--Schwarz gives
\[
\big\|\Lambda_H^{1/2}(I_k-M_+)\beta_H^\star\big\|_2^2
\ge
\lambda_-c_0^2\|\beta_H^\star\|_2^2
\ge
\frac{c_0^2\lambda_-}{\lambda_+}\,\Theta_H
\gtrsim\Theta_H
\asymp
\sum_{j\le k}\Theta_j\Big(\frac{a}{a+\lambda_j}\Big)^2,
\]
the final equivalence because \(a\asymp\lambda_j\) makes each
\(a/(a+\lambda_j)\asymp1\).  This is the positive-filter lower bound; the head
condition number \(\lambda_+/\lambda_-=O(1)\) is what makes the constant-factor
comparison hold coordinatewise.

For the endpoint, let \(d_\nu=|a-\nu|\).  On the tail-band subspace from
Lemma~\ref{lem:journal-multispike-gaussian-event},
\[
(K-\nu I_n)^{-2}
\succeq
\frac{P_B}{(d_\nu+Cw_k)^2}.
\]
The projected square-mass bound then yields
\[
\frac{\sigma_\varepsilon^2}{n}
\tr\{K_{2,T}(K-\nu I_n)^{-2}\}
\gtrsim
\frac{\sigma_\varepsilon^2S_k}
{n(d_\nu+Cw_k)^2}.
\tag{Endpoint square-mass variance}
\label{eq:ms-endpoint-tail-variance}
\]
If \(\nu>a\), admissibility and the existence of \(n-k\) tail-band
eigenvalues force \(d_\nu\lesssim w_k\), so this variance alone dominates
the infimum in \eqref{eq:multispike-endpoint-envelope}.

It remains to consider \(\nu\le a\), with \(b=a-\nu\).  When
\(b\lesssim w_k\), the same variance argument applies.  When
\(b\ge Cw_k\), the tail resolvent
\(B_\nu:=K_T-\nu I_n=E_T+bI_n\) is positive definite and
\[
C_\nu:=\frac1nX_H^\top B_\nu^{-1}X_H
\succeq0.
\]
Choose the numerical constant \(C\) large enough that
\(\|E_T\|_{\rm op}\le b/2\).  Then
\[
\|C_\nu\|_{\rm op}
\le
\|G_H\|_{\rm op}\|B_\nu^{-1}\|_{\rm op}
\le
\frac{C_1\lambda_+}{b}.
\]
The exact head map is \(C_\nu(I_k+C_\nu)^{-1}\), so the head signal residual
is \((I_k+C_\nu)^{-1}\beta_H^\star\).  Since \(C_\nu\succeq0\),
\[
\begin{aligned}
\left\|
\Lambda_H^{1/2}(I_k+C_\nu)^{-1}\beta_H^\star
\right\|_2^2
&\ge
\frac{\lambda_-}{(1+\|C_\nu\|_{\rm op})^2}
\|\beta_H^\star\|_2^2\\
&\ge
c\Theta_H\frac{b^2}{(b+\lambda_+)^2},
\end{aligned}
\]
where the last step uses
\(\Theta_H\le\lambda_+\|\beta_H^\star\|_2^2\) and
\(\lambda_+/\lambda_-=O(1)\).  This proves the required signal-bias term
without a head-Wishart expansion.  Combining it with
\eqref{eq:ms-endpoint-tail-variance}, and adding the resolved head-noise floor
\[
\frac{\sigma_\varepsilon^2}{n}
\tr\{K_{2,H}(K-\nu I_n)^{-2}\}
\gtrsim
\frac{\sigma_\varepsilon^2k}{n},
\]
proves the endpoint envelope.  For the last display, use
\((K-\nu I_n)^{-2}\succeq\|K\|_{\rm op}^{-2}I_n\),
\(\tr(K_{2,H})\asymp k\lambda_-^2\), and
\(\|K\|_{\rm op}\lesssim\lambda_+\).
\end{proof}

\begin{proof}[Proof of
Theorem~\ref{thm:journal-multispike-floor-critical}]
Apply Lemma~\ref{lem:journal-multispike-signed-upper} at
\(t=t_{\rm fc}=L_k/\lambda_-\).  By definition of \(L_k\),
\[
\Theta_He^{-2L_k}
\le
\frac{\sigma_\varepsilon^2S_k}{n\lambda_-^2}
=V_{T,k}.
\]
To make the asymptotic substitutions explicit, first note that
\[
\frac{S_k}{n\lambda_-^2}
\le
\frac{w_k^2}{\lambda_-^2}
=O(r_k^2)\longrightarrow0.
\]
The bounded-signal and bounded-noise assumptions therefore imply
\(L_k\to\infty\).  Put
\(\rho_k:=w_k/\lambda_-\) and \(s_{n,k}:=\chi_{n,k}\).
Since \(\rho_k=O(r_k)\), the first sharpness condition gives
\(w_kt_{\rm fc}=\rho_kL_k=o(1)\).
Moreover, the head-Wishart event gives
\[
\lambda_{\min}(G_H)
\ge
\lambda_-\left(1-C\chi_{n,k}\right).
\]
Thus the ideal head residual contributes
\[
\Theta_H
\exp\!\left[
-2L_k+
O\!\left(L_k\chi_{n,k}\right)
\right]
=(1+o(1))\Theta_He^{-2L_k},
\]
because \(L_ks_{n,k}\to0\).  Moreover,
\eqref{eq:ms-localized-head-map-error} and
\(e^{Cw_kt_{\rm fc}}=1+o(1)\) give
\[
\delta_H(t_{\rm fc})
\lesssim
\rho_ks_{n,k}+\rho_k^2(1+L_k),
\]
and hence
\[
\Theta_H\delta_H(t_{\rm fc})^2
\lesssim
\Theta_H
\left\{
\rho_k^2s_{n,k}^2+\rho_k^4(1+L_k^2)
\right\}
=o(\Theta_H\rho_k^2).
\]
Finally, the tail-noise term becomes
\[
\frac{\sigma_\varepsilon^2S_k}{n}
t_{\rm fc}^2e^{Cw_kt_{\rm fc}}
=(1+o(1))V_{T,k}L_k^2.
\]
Substitution in \eqref{eq:ms-localized-signed-envelope} proves
\[
R_{\rm sign}(a,t_{\rm fc})
\lesssim
\frac{\sigma_\varepsilon^2k}{n}
+
V_{T,k}(1+L_k^2)
+
\Theta_H\frac{w_k^2}{\lambda_-^2}.
\]
This proves the signed upper bound at the floor-critical level.
Lemma~\ref{lem:journal-multispike-comparators} gives the other rows and the
endpoint envelope.

For the simplified endpoint rate, put
\[
b_k^\star
=
\left(
\frac{\sigma_\varepsilon^2(S_k/n)\lambda_+^2}{\Theta_H}
\right)^{1/4}.
\]
Because \(S_k/(n\lambda_+^2)\to0\), we have
\(b_k^\star/\lambda_+\to0\).  If \(b\le b_k^\star\), then, using
\(w_k=o(b_k^\star)\), the variance term in
\eqref{eq:multispike-endpoint-envelope} is at least a constant multiple
of \(\sigma_\varepsilon\sqrt{\Theta_HS_k/n}/\lambda_+\).  If
\(b\ge b_k^\star\), the head-bias term has the same lower order.  This
proves \eqref{eq:multispike-endpoint-simplified}.
\end{proof}

\begin{proof}[Proof of
Corollary~\ref{cor:journal-multispike-flat-tail}]
For a flat tail,
\(T_k=d_T\lambda_T=na\) and
\(S_k=d_T\lambda_T^2=na\lambda_T\),
and
\[
w_k=\lambda_T+\sqrt{a\lambda_T}
\asymp\sqrt{a\lambda_T}.
\]
Moreover, \(r_k\to0\), and the head-frame condition implies \(k/n\to0\);
hence
\[
\frac{k w_k^2}{S_k}
\asymp\frac{k}{n}
\longrightarrow0.
\]
The balanced-endpoint condition is automatic here, because
\[
\frac{w_k}{b_k^\star}
\asymp
\left(\frac{\lambda_T}{\lambda_-}\right)^{1/4}
\longrightarrow0
\]
under \(a\asymp\lambda_-\), \(\lambda_+/\lambda_-=O(1)\), and the bounded
signal and noise scales.
These identities reduce the theorem's signed upper bound and endpoint
lower bound to the two displayed rates in the corollary.
The MP wall follows from
Proposition~\ref{prop:journal-mp-barrier}.  Division by the comparator
lower bounds in \eqref{eq:multispike-risk-bounds} gives the two ratios.
For the negative endpoint, the two terms on the right-hand side tend to zero
by \eqref{eq:multispike-flat-tail-total-separation}.  For uniform rescaling
and positive shrinkage, \(k/n\to0\), while
\[
\frac{\lambda_T}{\lambda_-}(1+L_k^2)
\asymp
\left(\frac{w_k}{\lambda_-}\right)^2(1+L_k^2)
=o(1)
\]
by the exposure condition in
\eqref{eq:multispike-flat-tail-sharpness}.
\end{proof}

\begin{proof}[Proof of
Corollary~\ref{cor:journal-multispike-power-law}]
Write \(H_{d,\alpha}=\sum_{\ell\le d}\ell^{-\alpha}\).  For
\(0<\alpha<1/2\),
\[
H_{d,\alpha}\asymp d^{1-\alpha},
\qquad
H_{d,2\alpha}\asymp d^{1-2\alpha}.
\]
The calibration \(c_n=na/H_{d_n,\alpha}\) gives
\[
\frac{S_k}{n}
=
\frac{c_n^2}{n}H_{d_n,2\alpha}
\asymp
a^2\frac{n}{d_n},
\qquad
\frac{\lambda_{k+1}}a
\asymp
\frac{n}{d_n^{1-\alpha}}.
\]
This proves \eqref{eq:multispike-power-law-scales}; bounded head signal and
noise scales give \(L_k\asymp\log(d_n/n)\).

The first inequality in
\eqref{eq:multispike-power-law-dimension} makes \(r_kL_k\to0\).
Since \(k=O(1)\) and \(\alpha<1/2\),
\[
\frac{k w_k^2}{S_k}
\lesssim
\frac{k\lambda_{k+1}^2}{S_k}+\frac{k}{n}
\asymp d_n^{2\alpha-1}+\frac1n
\longrightarrow0.
\]
Because \(d_n=n^q+O(1)\), we also have
\[
L_k\chi_{n,k}
\asymp
(\log n)\sqrt{\frac{\log n}{n}}
\longrightarrow0.
\]
The second gives
\[
\frac{w_k}{b_k^\star}
\lesssim
\frac{n/d_n^{1-\alpha}}{(n/d_n)^{1/4}}
+
\left(\frac nd_n\right)^{1/4}
\longrightarrow0.
\]
Substitution in the theorem proves
\eqref{eq:multispike-power-law-signed} and the square-root endpoint rate.
Indeed, with
\[
u_n:=\frac{n}{d_n^{1-\alpha}},
\qquad
v_n:=\sqrt{\frac n{d_n}},
\]
the theorem gives
\[
R_{\rm sign}(a,t_{\rm fc})
\lesssim
\frac{\sigma_\varepsilon^2k}{n}
+
\Theta_Hu_n^2
+
\{\Theta_H+\sigma_\varepsilon^2(1+L_k^2)\}v_n^2.
\]
The second inequality in
\eqref{eq:multispike-power-law-dimension} gives \(u_n^2=o(v_n)\), while
\(v_nL_k^2\to0\) because \(q>1\).  Thus these two terms are
\(o(\sqrt{n/d_n})\).  The head-noise term is negligible because
\[
\frac{k/n}{\sqrt{n/d_n}}
=O\!\left(n^{(q-3)/2}\right)
\longrightarrow0
\]
by \(q<3\).  The shape and positive-filter comparisons follow from
\eqref{eq:multispike-risk-bounds}.
\end{proof}

\begin{proof}[Proof of Corollary~\ref{cor:journal-multispike-discrete}]
For \(x\in[-Cw_k,C\lambda_+]\), \(\eta_m=t_{\rm fc}/m\), and
\(\eta_m(C\lambda_++Cw_k)\le1/2\), the scalar Euler estimate gives
\[
\left|
\eta_m\sum_{\ell=0}^{m-1}(1-\eta_mx)^\ell
-
\int_0^{t_{\rm fc}}e^{-sx}\,ds
\right|
\lesssim
\frac{t_{\rm fc}^2(|x|+w_k)}{m}e^{Cw_kt_{\rm fc}}.
\]
The displayed lower bound on \(m\), together with
\(\lambda_+/\lambda_-=O(1)\), \(w_k/\lambda_-\to0\), and
\(t_{\rm fc}=L_k/\lambda_-\), gives
\[
\eta_m(\widehat\mu_1+a)
\lesssim
\frac{w_k}{\lambda_-}\frac{L_k}{1+L_k^2}
=o(1)
\]
on the theorem event.  It therefore implies the preceding scalar small-step
condition for all sufficiently large \(n\).
Put \(D_m=h^{\rm disc}_{a,\eta_m,m}(K)-h_{a,t_{\rm fc}}(K)\).
Since the spectrum of \(K-aI_n\) lies in the displayed interval, functional
calculus and the stated lower bound on \(m\) give
\[
\|D_m\|_{\rm op}
\lesssim
\frac{L_k^2}{m\lambda_-}
\lesssim
\frac{w_k}{\lambda_-^2}.
\]
Consequently,
\[
\left\|
\frac1nX_H^\top D_mX_H
\right\|_{\rm op}
\lesssim\frac{w_k}{\lambda_-},
\qquad
\|h^{\rm disc}_{a,\eta_m,m}(K)P_H\|_{\rm op}
\lesssim\lambda_-^{-1}.
\]
The first bound adds at most
\(C\Theta_Hw_k^2/\lambda_-^2\) to the head bias, and the second preserves the
tail-signal leakage and head-noise bounds.  Finally, since both filters are
self-adjoint functions of \(K\),
\[
\{h^{\rm disc}_{a,\eta_m,m}(K)\}^2
\preceq
2h_{a,t_{\rm fc}}(K)^2+2D_m^2.
\]
Therefore the additional tail-noise contribution is at most
\[
\frac{C\sigma_\varepsilon^2}{n}
\|D_m\|_{\rm op}^2\tr(K_{2,T})
\lesssim
V_{T,k}\frac{w_k^2}{\lambda_-^2},
\]
which is below \(V_{T,k}(1+L_k^2)\).  Inserting these estimates into
\eqref{eq:ms-exact-risk} proves the refined signed upper bound.
\end{proof}

\bibliographystyle{plainnat}
\bibliography{adaptive_ridge}

\end{document}